%% file: main.tex
\documentclass{article}

\PassOptionsToPackage{numbers,compress}{natbib}
\usepackage[preprint]{neurips_2026}
\usepackage[utf8]{inputenc}
\usepackage[T1]{fontenc}
\usepackage{hyperref}
\usepackage{url}
\usepackage{enumitem}
\usepackage{graphicx}
\usepackage{amsmath,amssymb,amsthm}

\usepackage{algorithm}
\usepackage{algorithmic}
\usepackage{booktabs}
\usepackage{float}
\usepackage{tikz}
\usetikzlibrary{positioning,arrows.meta,backgrounds,calc,fit}
\usepackage{xcolor}
\definecolor{tcorrect}{HTML}{2E7D32}
\definecolor{twrong}{HTML}{C62828}
\definecolor{tneutral}{HTML}{616161}
\definecolor{tboxbg}{HTML}{F5F5F5}
\definecolor{tboxrule}{HTML}{BDBDBD}

\newcommand{\trace}{\textsc{TRACE}}

\newtheorem{proposition}{Proposition}[section]
\newtheorem{lemma}{Lemma}[section]
\newtheorem{theorem}{Theorem}[section]
\theoremstyle{remark}
\newtheorem{remark}{Remark}[section]


\title{\trace{}: Trajectory Correction from Cross-layer Evidence for Hallucination Reduction}
\author{Tej Sanibh Ranade\\
  Independent Researcher\\
  \texttt{tsranade1@gmail.com}}

\begin{document}
\maketitle

\begin{abstract}
Hallucination correction is not a one-direction problem. We show that intermediate layers are neither uniformly more truthful than final layers nor uniformly less trustworthy. Yet hallucination reduction is usually instantiated through one fixed intervention form: contrast one layer against another, steer along a truthfulness direction, or defer to external evidence. This framing is structurally incomplete. Cross-layer factual evidence does not evolve uniformly: in some failures truthful support is present internally and later suppressed, whereas in others candidate competition remains genuinely multi-directional across depth, so no single signed scalar family is generally sufficient. We introduce \emph{Trajectory Correction from Cross-layer Evidence for Hallucination Reduction} (\trace{}), a deterministic, training-free algorithm which corrects hallucinations at inference time by deriving both the corrective layer and the appropriate correction operator from each input's cross-layer candidate trajectory inside the LLM's own forward pass. Under one frozen hyperparameter setting, \trace{} selects among scalar reversal, earlier-state recovery, and candidate-space correction using only model-internal evidence. Evaluated as a single universal algorithm across 15 models, 8 model families, and 3 factuality benchmarks, \trace{} improves every evaluation cell, yielding mean gains of \textbf{+12.26} MC1 points and \textbf{+8.65} MC2-style points with \textbf{no regressions}, with gains reaching \textbf{+47.20} MC1 and \textbf{+43.38} MC2-style points. The method uses no labels, retrieval, pretraining, finetuning, or per-model calibration.
\end{abstract}

\section{Introduction}

Hallucinations remain a central limitation of large language models (LLMs), despite rapid progress in scale and post-training methods such as instruction tuning and human-preference alignment~\citep{wang2024factualitysurvey}. Benchmarks such as TruthfulQA and HaluEval continue to expose confident factual errors in otherwise strong models~\citep{lin2022truthfulqa,li2023halueval}. A pretrained model may already contain corrective evidence that is not faithfully expressed at the output layer. Two empirical findings make such control plausible. First, truthful support is often represented internally more strongly than it is expressed at the final layer~\citep{burns2023latent,belrose2023tuned,chen2024inside}. Second, internal activations carry usable information about factuality and hallucination risk~\citep{azaria2023lying,li2023iti}. Our central claim is that intermediate layers are neither uniformly more truthful nor uniformly less trustworthy than final layers; usefulness depends on how candidate support evolves across depth. Depth alone is not a correction rule. These observations have motivated three main intervention lines. Retrieval and verification systems incorporate external evidence~\citep{lewis2020rag,gao2023rarr}. Activation-space methods intervene along directions associated with truthfulness, with ITI as the canonical inference-time example and later variants adapting intensity or steering policy~\citep{li2023iti,bayat2024lito}. Layerwise decoding methods contrast or aggregate intermediate and final-layer signals, beginning with DoLa and extending through SLED, token-wise cross-layer entropy, SH2, DeLTa, and ActLCD~\citep{chuang2024dola,zhang2024sled,wu2025end,kai2024sh2,he2025delta,zhang2025actlcd}. This literature establishes that internal signals are usable, but it also exposes a structural limit. Retrieval and verification depend on external evidence; fixed-intensity activation interventions are context-dependent~\citep{bayat2024lito}; and fixed-form layer-contrast decoding can regress materially on some model-task pairs or architectures, for example reducing GLM4 TruthfulQA \%T$\ast$I from $56.08$ to $48.44$ and LLaMA3.1 LongFact F1@128 from $88.39$ to $84.42$ in ActLCD, while DeLTa reports DoLa reducing Mistral-7B TruthfulQA \%True$\ast$\%Info from $53.8$ to $42.9$~\citep{zhang2025actlcd,he2025delta}. The open problem is not whether internal signals exist, but whether they can support one inference-time correction procedure as cross-layer behavior changes. A fixed intervention can succeed only when the current trajectory matches its assumptions.

What these methods largely leave implicit is the \emph{cross-layer candidate trajectory}: the depth-indexed evolution of candidate preference for a given prompt. Figure~\ref{fig:intro-traces} shows representative instances. In Panel~(a), the premise asks whether a gin and tonic and a Paloma are both cocktails based on tequila. The truthful completion \emph{yes} dominates almost the entire network, so truthful support is already internal; the error appears only when the final layers redirect probability to \emph{not both drinks}. In Panel~(b), the premise asks where Vestfold was born. Here the truthful completion \emph{Vestfold} is strongest in the earliest layers, but later layers steadily erase that support until the model settles on \emph{born in Oslo}. In Panel~(c), the premise asks what tastes so good that cats ask for it by name. The truthful completion \emph{cats can't ask} becomes preferred at an intermediate layer, but several false alternatives remain active throughout the forward pass, and the final layer reallocates probability to the cat-myth answer \emph{catnip}. These cases are representative rather than exhaustive, but they expose two recurring facts: trustworthy evidence is not tied to a fixed depth, and the admissible correction is not uniformly scalar. Mechanistic analyses of factual failures are consistent with this distinction: some hallucinations suppress truthful support that is already present internally, whereas others preserve unresolved competition among alternatives deep into the network~\citep{jiang2024knownfacts,yu2024mechanistic}. The failure of fixed-form methods is therefore not only insufficient tuning; it is a mismatch between correction class and trajectory structure. A universal method must therefore infer, from the trajectory itself, both where reliable evidence resides and which correction regime it supports.

\begin{figure}[t]
    \centering
    \vspace{-0.5em}
    \includegraphics[width=\textwidth]{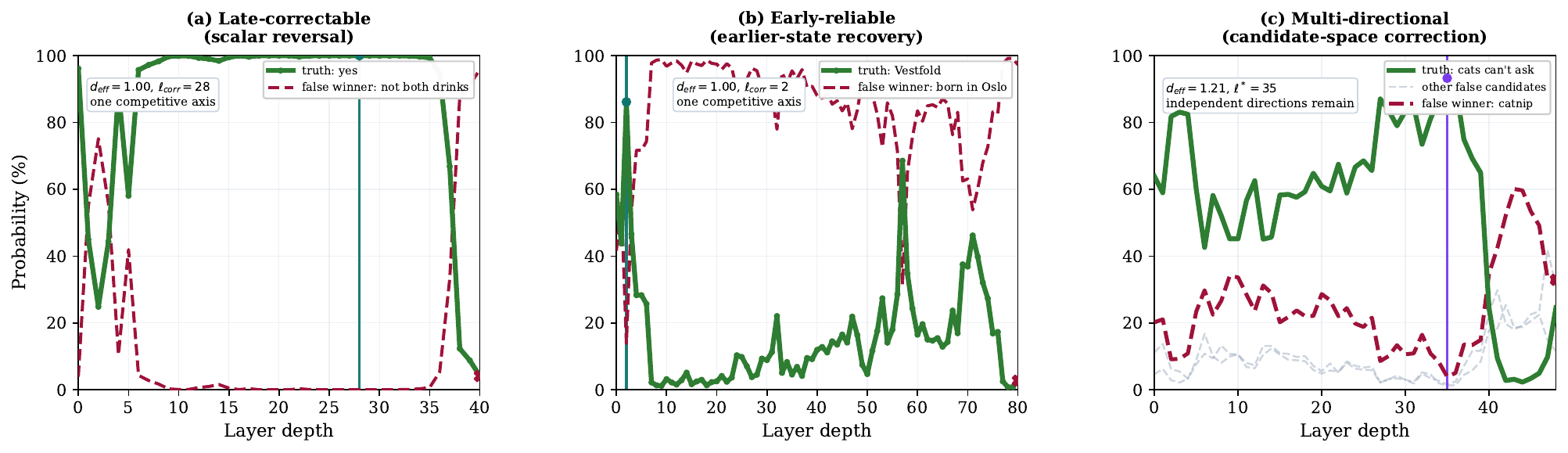}
    \vspace{-0.6em}
    \caption{\textbf{Representative TRACE correction regimes.} Green: truthful candidate; red: false final-layer winner; dashed gray: other false candidates. Vertical line: TRACE's selected layer ($\ell_{\mathrm{corr}}$ in scalar panels, $\ell^\ast$ in the candidate-space panel).}
    \label{fig:intro-traces}
\end{figure}
\vspace{0pt}

We formalize this principle in \emph{Trajectory Correction from Cross-layer Evidence for Hallucination Reduction} (TRACE), a deterministic, training-free inference-time algorithm over cross-layer candidate trajectories. TRACE induces its intervention from trajectory structure: scalar when competition is effectively one-directional, with reversal versus earlier-state recovery determined by the locus of reliable evidence, and candidate-space when multiple competitive directions remain active. The procedure requires no labels, retrieval, pretraining, finetuning, or per-model calibration.

For a practical inference-time method, the relevant empirical bar is frozen transfer without silent regressions, not average improvement on a favored subset. We evaluate TRACE as a single universal algorithm, under one fixed setting, on a 15-model, 3-benchmark factuality grid spanning 8 model families. TRACE improves every model-benchmark cell, yielding mean gains of \textbf{+12.26} MC1 points and \textbf{+8.65} MC2-style points with \textbf{0/45 regressions}; individual cells reach gains of \textbf{+47.20} MC1 and \textbf{+43.38} MC2-style points. The result is evidence that the forward computation of pretrained language models contains enough structured evidence to sustain a single hallucination-reduction algorithm, provided that the intervention is induced by the trajectory itself. This paper's main contributions are:
\begin{enumerate}[leftmargin=1.6em,itemsep=0.08em,topsep=0.18em,label=\arabic*.]
    \item We identify the cross-layer candidate trajectory as the primary object for inference-time hallucination reduction, replacing fixed probe layers and fixed steering directions with a trajectory-level formulation.
    \item We develop a structural theory of correction regimes that separates effectively one-directional from genuinely multi-directional competition and further distinguishes late-correctable from early-reliable scalar regimes.
    \item We prove formal results characterizing the admissible correction families induced by trajectory structure, including the reduction to scalar correction in the one-directional regime and the necessity of candidate-space correction when multiple competitive directions remain active.
    \item We instantiate these results in TRACE, a deterministic training-free algorithm that induces its correction operator from the current trajectory and localizes trustworthy internal evidence without labels, retrieval, pretraining, finetuning, or per-model calibration.
\end{enumerate}

\section{TRACE Method}

Two structural quantities determine the correction operator that \trace{} applies to each item. The item-side quantity is the \emph{cross-layer candidate trajectory} $S(x)\in\mathbb R^{n\times(L+1)}$: a depth-indexed table of length-normalized log-probabilities obtained from one candidate-conditioned forward pass per completion. The model-side quantity is a weights-only invariant $I(M)$, computed once from checkpoint tensors. The candidate space is $\mathbb R^n$, and a \emph{correction operator} is a deterministic map from $(S(x),I(M))$ to a vector in $\mathbb R^n$ whose argmax is returned.

The first quantity is item-side. As the residual stream evolves through depth, the contest between candidates is either one-dimensional (every layer expresses the same preference pattern across candidates, varying only in strength and sign) or genuinely multi-dimensional, with different layers expressing orderings that no single scaling can reconcile. We summarize this geometry with the \emph{effective trajectory dimension} $d_{\mathrm{eff}}(x)$: the participation ratio of the candidate-space Gram matrix of the centered trajectory, equal to one in the single-axis case and growing toward the rank of the trajectory when several directions remain active. The second is model-side. Architectures differ in whether their late layers amplify the evidence accumulated in the mid-block residual stream cleanly into the output, or drift past a more reliable earlier state. We summarize this with a \emph{weights-only invariant} $I(M)$: a dimensionless ratio placing the rescaling strength of the final layer norm and the row-norm concentration of the mid-block output projection against the row-norm concentrations of the early key and value projections. $I(M)$ is computed once from $M$'s weight tensors, with no inference and no data.

The structural theorem of Section~\ref{sec:deff} (Theorem~\ref{thm:operators}) shows that $d_{\mathrm{eff}}(x)$ partitions items into two regimes (scalar and candidate-space) with disjoint admissible operator classes. Inside the scalar regime, $I(M)$ further selects between two operators (signed mixing and earliest-state fallback), giving three operators in total. Section~\ref{sec:operators} constructs them; Algorithm~\ref{alg:trace} assembles the inference rule.

\subsection{Cross-layer candidate trajectory}

We construct $S(x)$ by reading every layer of $M$ through the model's own output head, the same one that produces its final logits. Concretely, we apply the final normalization $\mathcal N$~\citep{ba2016layernorm,zhang2019rmsnorm} and the unembedding matrix $W_U$ to the residual-stream state at every depth. This is the standard \emph{logit-lens} construction~\citep{belrose2023tuned}; it yields an interpretable token distribution at each layer without modifying $M$. The resulting layerwise scores describe how candidate preference evolves with depth, and \trace{} consumes them as a depth-indexed object, not as $L+1$ independent predictors. Formally, with $\mathcal{C}(x)=\{c_1,\dots,c_n\}$ and each candidate $c_i$ tokenized as $(v_{i,1},\dots,v_{i,m_i})$, run $M$ on the concatenation $(x,c_i)$ and let $\mathbf h_{\ell,t}(x,c_i)\in\mathbb R^d$ denote the residual-stream state at depth $\ell\in\{0,\dots,L\}$ at the input position from which $M$ predicts $v_{i,t}$. All such states are recorded from one forward pass per candidate. The length-normalized layerwise candidate score is
\begin{equation}
s_{\ell,i}(x)=\frac{1}{m_i}\sum_{t=1}^{m_i}
\log\Bigl[\operatorname{softmax}\!\bigl(W_U\mathcal N(\mathbf h_{\ell,t}(x,c_i))\bigr)\Bigr]_{v_{i,t}}.
\label{eq:layerwise-score}
\end{equation}
Write $\mathbf s_\ell(x)=(s_{\ell,1}(x),\dots,s_{\ell,n}(x))^\top$ and stack across depth into the \emph{cross-layer candidate trajectory}
\begin{equation}
S(x)=\bigl[\mathbf s_0(x)\ \mathbf s_1(x)\ \cdots\ \mathbf s_L(x)\bigr]\in\mathbb R^{n\times(L+1)},\qquad
\mathbf b(x)=\mathbf s_L(x).
\label{eq:trajectory}
\end{equation}
$\mathbf b(x)$ is the base prediction; length normalization aligns scores across candidates of different lengths. \trace{} adds no learned parameters to this readout.

\subsection{Effective trajectory dimension}
\label{sec:deff}

Given the trajectory $S(x)$, our first task is to detect whether candidate competition across depth lives along a single signed axis or genuinely spans several directions. This distinction determines which class of correction operator is admissible (Theorem~\ref{thm:operators} below). Only \emph{relative} motion across candidates can change the argmax, so we subtract the per-layer mean to project each score vector onto the \emph{centered candidate space} $\{v\in\mathbb R^n:\mathbf 1^\top v=0\}$, then restrict to a fixed interior window of layers:
\begin{equation}
\mathcal L_{\mathrm{mid}}=\{\ell:\ell_-\le \ell<\ell_+\},\qquad
S_{\mathrm{mid}}(x)=S(x)\big|_{\mathcal L_{\mathrm{mid}}},
\label{eq:mid-window}
\end{equation}
\begin{equation}
X(x)=S_{\mathrm{mid}}(x)-\tfrac1n\mathbf 1\mathbf 1^\top S_{\mathrm{mid}}(x),\qquad
C(x)=X(x)X(x)^\top,
\label{eq:centered-trajectory}
\end{equation}
with $\ell_-=\lfloor\rho_-L\rfloor$ and $\ell_+=\lfloor\rho_+L\rfloor$ ($0<\rho_-<\rho_+\le 1$ hyperparameters); the interior excludes both the embedding readout (where the logit-lens projects through tied input/output weights and so reads the input token rather than evolved candidate evidence) and the final layer (which is $\mathbf b(x)$ itself). $X(x)$ is the centered trajectory matrix and $C(x)$ its candidate-space Gram matrix. Letting $\{\lambda_j(x)\}_{j=1}^{r}$ denote the nonzero eigenvalues of $C(x)$ (with $r=\operatorname{rank}(C(x))$), the trace identities $\operatorname{tr}C(x)=\sum_j\lambda_j(x)$ and $\operatorname{tr}(C(x)^2)=\sum_j\lambda_j(x)^2$ give the \emph{effective trajectory dimension}
\begin{equation}
d_{\mathrm{eff}}(x)=\frac{\bigl(\operatorname{tr}C(x)\bigr)^2}{\operatorname{tr}(C(x)^2)}=\frac{\bigl(\sum_{j=1}^{r}\lambda_j(x)\bigr)^2}{\sum_{j=1}^{r}\lambda_j(x)^2},
\label{eq:deff}
\end{equation}
the \emph{participation ratio} of the spectrum: it equals $1$ exactly when only one eigenvalue is nonzero, equals $r$ when all $r$ nonzero eigenvalues are equal, and otherwise interpolates between. $d_{\mathrm{eff}}(x)$ is therefore the smooth count of effectively nonzero candidate-space directions in which the centered trajectory varies. The one-directional regime ($d_{\mathrm{eff}}(x)=1$) covers every contest in which all layers share the same preferred-vs-dispreferred partition (any 2-candidate contest is automatic; Proposition~\ref{prop:deff} below); the multi-directional regime is reserved for contests in which several candidates pull along independent directions.

\begin{proposition}[Basic properties of $d_{\mathrm{eff}}$]
\label{prop:deff}
Assume $X(x)\neq 0$ and let $r=\operatorname{rank}(X(x))$. Then
\[
1\le d_{\mathrm{eff}}(x)\le r\le n-1.
\]
Moreover, $d_{\mathrm{eff}}(x)=1$ if and only if $r=1$. In particular, every binary candidate set satisfies $d_{\mathrm{eff}}(x)=1$.
\end{proposition}

\begin{proof}
Let $\lambda_1,\dots,\lambda_r>0$ be the nonzero eigenvalues of $C(x)$. Then $d_{\mathrm{eff}}(x)=(\sum_j\lambda_j)^2/\sum_j\lambda_j^2$. Since $\lambda_j\ge 0$, $(\sum_j\lambda_j)^2\ge\sum_j\lambda_j^2$, so $d_{\mathrm{eff}}(x)\ge 1$, with equality iff exactly one $\lambda_j$ is nonzero, i.e.\ $r=1$. By Cauchy--Schwarz, $(\sum_j\lambda_j)^2\le r\sum_j\lambda_j^2$, so $d_{\mathrm{eff}}(x)\le r$. Every column of $X(x)$ sums to zero, hence $\operatorname{row}(X(x))\perp \mathbf 1$ and $r\le n-1$. When $n=2$, this forces $r=1$ whenever $X(x)\neq 0$.
\end{proof}

Proposition~\ref{prop:deff} ties $d_{\mathrm{eff}}$ to the rank of the centered trajectory, so the one-directional regime is identified by a geometric quantity rather than a candidate-counting heuristic. Let $\mathcal U(x)=\operatorname{col}(X(x))\subseteq \mathbb R^n$ be the candidate-space subspace supported by the trajectory; the next result shows that $\mathcal U(x)$ determines the admissible correction class.

\begin{theorem}[Geometry forces operator classes]
\label{thm:operators}
Fix a prompt $x$ and suppose, in addition to the trajectory $S(x)$, we have access to a candidate-summary vector $\mathbf t(x)\in\mathbb R^n$ (one log-probability per candidate, derived from the same forward pass).
\begin{enumerate}[leftmargin=1.35em,itemsep=0.04em,topsep=0.08em,label=(\roman*)]
    \item If $d_{\mathrm{eff}}(x)=1$, then $\mathcal U(x)=\operatorname{span}\{\mathbf u(x)\}$ for some nonzero $\mathbf u(x)$. In particular, every affine score of the form $\alpha\mathbf b(x)+\beta\mathbf t(x)+\gamma\mathbf 1$ with $\alpha+\beta>0$ has the same argmax as $(1-\lambda)\mathbf b(x)+\lambda\mathbf t(x)$, where $\lambda=\beta/(\alpha+\beta)$.
    \item There exist $\mathbf b,\mathbf t,\mathbf q\in\mathbb R^3$ such that the centered matrix with columns $\mathbf b-\bar b\mathbf 1$ and $\mathbf q-\bar q\mathbf 1$ has rank $2$ (where $\bar v=\tfrac1n\mathbf 1^\top v$), and $\arg\max_i q_i$ cannot be written as $\arg\max_i[(1-\lambda)b_i+\lambda t_i]$ for any $\lambda\in\mathbb R$.
\end{enumerate}
Consequently, no correction rule restricted to scalar mixtures of $(\mathbf b(x),\mathbf t(x))$ can be universal over domains containing both $d_{\mathrm{eff}}(x)=1$ and $d_{\mathrm{eff}}(x)>1$ items; a second, candidate-space operator class is necessary.
\end{theorem}

\begin{proof}
By Proposition~\ref{prop:deff}, $d_{\mathrm{eff}}(x)=1$ implies $\operatorname{rank}(X(x))=1$, so $X(x)=\mathbf u(x)\mathbf w(x)^\top$ for nonzero $\mathbf u(x),\mathbf w(x)$ and therefore $\mathcal U(x)=\operatorname{span}\{\mathbf u(x)\}$. This proves the first sentence of (i). For the affine claim, define $\mathbf z=\alpha\mathbf b+\beta\mathbf t+\gamma\mathbf 1$ and $\lambda=\beta/(\alpha+\beta)$. Then for every pair $(i,j)$,
\[
z_i-z_j
=
\alpha(b_i-b_j)+\beta(t_i-t_j)
=
(\alpha+\beta)\bigl[(1-\lambda)(b_i-b_j)+\lambda(t_i-t_j)\bigr].
\]
Because $\alpha+\beta>0$, pairwise signs are preserved, so the argmax is identical.

For (ii), take $\mathbf b=(3,2,1)^\top$ and $\mathbf t=(2,3,1)^\top$. Every scalar mixture is $(3-\lambda,2+\lambda,1)^\top$, so candidate $3$ would require $\lambda>2$ and $\lambda<-1$ simultaneously; hence candidate $3$ is never the argmax of any scalar mixture. Now take $\mathbf q=(1,1,5)^\top$, whose argmax is candidate $3$. Finally, $(1,0,-1)^\top=\mathbf b-2\mathbf 1$ and $(-\tfrac43,-\tfrac43,\tfrac83)^\top=\mathbf q-\tfrac73\mathbf 1$ are linearly independent, so a centered trajectory containing both has rank at least $2$.
\end{proof}

Theorem~\ref{thm:operators} fixes the architecture of \trace{}. Part~(i) shows that rank-one items reduce the scalar family to one free coefficient $\lambda$ over the pair $(\mathbf b,\mathbf t)$. Part~(ii) does not claim that every multi-directional item must flip under the same override; it establishes the stronger universal fact that scalar mixtures cease to be complete beyond rank one. \trace{} therefore branches first on $d_{\mathrm{eff}}(x)$ and instantiates, in the multi-directional regime, a deterministic candidate-space operator that is not confined to the scalar family.

\subsection{Regime-specific correction operators}
\label{sec:operators}

Theorem~\ref{thm:operators} partitions items into two regimes by effective dimension. The \emph{candidate-space regime} ($d_{\mathrm{eff}}(x)>\tau_{\dim}$) covers items whose contest is genuinely multi-directional; the \emph{scalar regime} ($d_{\mathrm{eff}}(x)\le\tau_{\dim}$) covers every contest that collapses to a single signed axis (any 2-candidate contest is automatically here by Proposition~\ref{prop:deff}). We now construct one operator for the candidate-space regime and two for the scalar regime (distinguished by $I(M)$), together with their gates.

\paragraph{Candidate-space operator (multi-directional regime).}
When $d_{\mathrm{eff}}(x)>\tau_{\dim}$, scalar mixtures are no longer universally complete, so \trace{} instantiates its second operator class directly in candidate space: it identifies the layer at which $M$'s candidate ranking is sharpest and uses that layer's full per-candidate distribution as the correction. To formalize ``sharpest'' we summarize each layer with three statistics: the candidate-restricted softmax, the top-two margin, and the entropy.
\begin{equation}
\pi_{\ell,i}(x)=\frac{e^{s_{\ell,i}(x)}}{\sum_{j=1}^n e^{s_{\ell,j}(x)}},\quad
m_\ell(x)=s_{\ell,(1)}(x)-s_{\ell,(2)}(x),\quad
H_\ell(x)=-\!\sum_{i=1}^n \pi_{\ell,i}(x)\log \pi_{\ell,i}(x),
\label{eq:layer-stats}
\end{equation}
where $s_{\ell,(1)}(x)\ge\dots\ge s_{\ell,(n)}(x)$ are the sorted components of $\mathbf s_\ell(x)$. The decisive layer maximizes sharpness per unit entropy,
\begin{equation}
D_\ell(x)=\frac{m_\ell(x)}{H_\ell(x)+\varepsilon_H},\quad
\ell^\ast(x)=\arg\max_{\ell\in\{1,\dots,L\}}D_\ell(x),\quad
q_i(x)=\log \pi_{\ell^\ast(x),i}(x),
\label{eq:decisive-layer}
\end{equation}
where $\varepsilon_H>0$ smooths division by near-zero entropy (no natural boundary). $D_\ell$ is a dimensionless sharpness statistic: high when the top-one gap is large relative to the residual uncertainty of the layerwise choice distribution. The argmax over $\ell\in\{1,\dots,L\}$ excludes the embedding readout for the same reason as in Section~\ref{sec:deff}. The candidate-space proposal is the layer-$\ell^\ast$ log-probability vector $\mathbf u_{\mathrm{md}}(x):=\mathbf q(x)=\log\boldsymbol\pi_{\ell^\ast(x)}$.

The candidate-space operator is gated before it overrides the base. With $i_b(x)=\arg\max_i b_i(x)$, $i_{\mathrm{md}}(x)=\arg\max_i u_{\mathrm{md},i}(x)$, and $\delta_r>0$ a numerical floor preventing division by zero when the final-layer margin vanishes (no natural boundary),
\begin{equation}
g_{\log r}(x)=\log\frac{\max_{\ell\in\mathcal L_{\mathrm{mid}}} m_\ell(x)}{\max(|m_L(x)|,\delta_r)},\qquad
g_H(x)=\frac{H_L(x)}{\log n},
\label{eq:md-scalars}
\end{equation}
\begin{equation}
\kappa_{\mathrm{md}}(x)=
\mathbb I[i_{\mathrm{md}}(x)\neq i_b(x)]\,
\mathbb I[g_{\log r}(x)>\tau_{\log r}]\,
\mathbb I[g_H(x)>\tau_H].
\label{eq:md-gate}
\end{equation}
The three indicators address three independent failure modes: the override must change the ranking, some mid-window layer must be sharper than the final, and the final layer must be sufficiently uncertain. All three are dimensionless; the log-margin ratio is scale-invariant, while the normalized entropy term intentionally measures sharpness on the model's native output scale. Falling back to $\mathbf b(x)$ when any indicator fails is the explicit abstention of \trace{} in the candidate-space regime.

\paragraph{Scalar operators (one-directional regime).}
When $d_{\mathrm{eff}}(x)\le \tau_{\dim}$, Theorem~\ref{thm:operators}(i) reduces the correction family to one signed scalar $\lambda$ acting on a fixed pair of score vectors $(\mathbf b(x),\mathbf t(x))$. Here $\mathbf b(x)=\mathbf s_L(x)$ is the final-layer answer score, while $\mathbf t(x)=\mathcal T(S(x);\theta_{\mathcal T})\in\mathbb R^n$ is a second answer-level view extracted from the same trajectory. Let $A,G\subseteq\{0,\dots,L\}$ denote fixed anchor and feature depth sets, let $(w_\ell)_{\ell\in A}$ be nonnegative anchor weights with $\sum_{\ell\in A}w_\ell=1$, let $z_{\ell,r}(x)\in\mathbb R^{|V|}$ be the vocabulary logit vector read from layer $\ell$ at continuation position $r$ with entry $z_{\ell,r}(v)$ at token $v$, and let $p_{\ell,r}(v)=\log\operatorname{softmax}(z_{\ell,r}(x))_v$. The role of $\mathcal T$ is to reread the forward pass tokenwise: it keeps tokens that recur across $A$, measures whether their support grows coherently across $G$, mixes that cross-layer summary back with the final logits, and then aggregates the result to one score per candidate. $\mathcal T$ is structurally necessary, not an optional refinement: Theorem~\ref{thm:operators}(i) shows that without a second answer-level reading not collinear with $\mathbf b$, the entire scalar correction family collapses to a one-parameter blend on rank-one items. $\mathcal T$ supplies that second reading by scoring tokens for cross-depth coherence, reusing the candidate-conditioned forward pass already used for $\mathbf b(x)$ with no auxiliary supervision. For candidate token $v$ at continuation position $r$, define
\begin{equation}
\Omega_r(x)=\Bigl\{v:\sum_{\ell\in A}\mathbb I[v\in \operatorname{TopK}_k(z_{\ell,r}(x))]\ge r_\Omega\Bigr\},
\quad
\bar z_r(v)=\sum_{\ell\in A} w_\ell z_{\ell,r}(v),
\label{eq:trace-omega}
\end{equation}
where $\Omega_r(x)$ keeps tokens recurring across at least $r_\Omega$ anchors and $\bar z_r$ is their weighted anchor average. TRACE then computes trajectory features from $(p_{\ell,r}(v))_{\ell\in G}$,
\begin{equation}
h_r(v)=a_s[\operatorname{slope}_r(v)]_+ + a_j[\operatorname{jump}_r(v)]_+ + a_c[\operatorname{curv}_r(v)]_+ .
\label{eq:trace-h}
\end{equation}
so positive slope, jump, and curvature reward tokens whose support is strengthening across depth. After relative normalization of $h_r(v)$ over $\Omega_r(x)$ when nonempty, TRACE forms
\begin{equation}
\alpha_r(v)=\lambda_0+(1-\lambda_0)\sigma\!\bigl(\gamma(\widetilde h_r(v)-\tfrac12)\bigr),
\qquad
z^{\mathcal T}_r(v)=\bigl(1-\alpha_r(v)\bigr)z_{L,r}(v)+\alpha_r(v)\bar z_r(v),
\label{eq:trace-tam}
\end{equation}
and emits one candidate score per completion,
\begin{equation}
t_i(x)=\frac{1}{m_i}\sum_{r=1}^{m_i}\log\operatorname{softmax}\!\bigl(z^{\mathcal T}_r\bigr)_{v_{i,r}}.
\label{eq:trace-t}
\end{equation}
Thus $\mathbf t(x)$ is a deterministic answer-level score vector computed from the same candidate-conditioned forward passes as $S(x)$, not an auxiliary learned probe; Appendix~\ref{app:scorer} records the full concrete instantiation.

The scalar regime admits two distinct correction operators, and which one applies is a property of the \emph{model}, not the item. The first is signed mixing of $(\mathbf b(x),\mathbf t(x))$, suitable when $\mathbf t(x)$ is a reliable summary; the second is fallback to the earliest-layer score $\mathbf s_0(x)$, suitable when the model's late layers have drifted past a more reliable earlier state, dragging late-weighted summaries like $\mathbf t(x)$ along with them. We distinguish the two cases with a \emph{weights-only invariant} $I(M)$ built from row-norm statistics of the model's attention and normalization weights. For any matrix $A\in\mathbb R^{r\times c}$, let
\begin{equation}
\operatorname{rcv}(A)=\frac{\operatorname{sd}_{1\le j\le r}\|A_{j:}\|_2}{\operatorname{mean}_{1\le j\le r}\|A_{j:}\|_2}.
\label{eq:rcv}
\end{equation}
Concretely, with structural depths $e=\lfloor\rho_e L\rfloor$ and $m=\lfloor\rho_m L\rfloor$ ($0<\rho_e<\rho_m<1$), let $W_K^{(e)},W_V^{(e)}\in\mathbb R^{r_K\times d}$, $W_V^{(m)},W_O^{(m)}\in\mathbb R^{r_V\times d}$ be the key, value, and output projections at those depths, and let $\mathbf w_{\mathcal N}\in\mathbb R^d$ be the final-norm weight (dimension $d_{\mathcal N}=d$). Define
\begin{equation}
\begin{aligned}
\phi_{\mathcal N}(M)&=\frac{\|\mathbf w_{\mathcal N}\|_1}{d_{\mathcal N}},&
\phi_K(M)&=\operatorname{rcv}(W_K^{(e)}),\\
\phi_V(M)&=\frac{\operatorname{rcv}(W_V^{(e)})}{\operatorname{rcv}(W_V^{(m)})},&
\phi_O(M)&=\operatorname{rcv}(W_O^{(m)}),
\end{aligned}
\label{eq:phi}
\end{equation}
and combine them into the dimensionless ratio
\begin{equation}
I(M)=\frac{\phi_{\mathcal N}(M)\phi_O(M)}{\phi_K(M)\phi_V(M)}.
\label{eq:invariant}
\end{equation}
The ratio in~\eqref{eq:invariant} places late evidence amplification (numerator: $\phi_{\mathcal N}\phi_O$) over early routing dominance (denominator: $\phi_K\phi_V$), with natural boundary $\tau_I=1$; Appendix~\ref{app:invariant} gives the exact extraction details across architectures.

Let $\Delta_2(\mathbf u)=u_{(1)}-u_{(2)}$ denote the top-two margin and $i_b(x)=\arg\max_i b_i(x)$ the base argmax. The sharpness check $\Delta_2(\mathbf t(x))>\Delta_2(\mathbf b(x))$ tests whether the trajectory carries more disambiguation power than the base. When it does, the base-top shift $t_{i_b(x)}(x)-b_{i_b(x)}(x)$ decides the sign of the mixture: positive shift means the trajectory reinforces the base's preferred candidate (\emph{trust}, $\lambda=+\eta$), nonpositive shift means the trajectory pulls against it (\emph{reverse}, $\lambda=-\eta$). When the trajectory is not sharper than the base, the operator abstains:
\begin{equation}
\lambda(x)=
\begin{cases}
+\eta, & \Delta_2(\mathbf t(x))>\Delta_2(\mathbf b(x))\ \text{and}\ t_{i_b(x)}(x)-b_{i_b(x)}(x)>0,\\
-\eta, & \Delta_2(\mathbf t(x))>\Delta_2(\mathbf b(x))\ \text{and}\ t_{i_b(x)}(x)-b_{i_b(x)}(x)\le 0,\\
0, & \text{otherwise,}
\end{cases}
\label{eq:lambda}
\end{equation}
which yields the late-correctable scalar score
\begin{equation}
\mathbf u_{\mathrm{sc}}(x)=\bigl(1-\lambda(x)\bigr)\mathbf b(x)+\lambda(x)\mathbf t(x).
\label{eq:usc}
\end{equation}
Theorem~\ref{thm:operators}(i) reduces the scalar regime to one free coefficient. TRACE resolves that coefficient with a signed rule and a global magnitude hyperparameter $\eta\in(0,1]$; its concrete value is fixed once for all experiments and reported in Section~\ref{sec:experiments}.

For $I(M)\le \tau_I$, $\mathbf t(x)$ is unreliable, so \trace{} suppresses scalar mixing. The earliest-state fallback fires only when the final layer is itself under-decisive ($\max_i b_i(x)<\gamma$, where $\gamma=\log e^{-1}=-1$ is the natural confidence floor, i.e.\ per-token probability below $1/e$); otherwise the base prediction is returned:
\begin{equation}
\mathbf u_{\mathrm{early}}(x)=
\begin{cases}
\mathbf s_0(x), & \max_i b_i(x)<\gamma,\\
\mathbf b(x), & \text{otherwise.}
\end{cases}
\label{eq:uearly}
\end{equation}

\subsection{Algorithm}

Let $\Theta=(\rho_-,\rho_+,\rho_e,\rho_m,\tau_{\dim},\tau_I,\tau_{\log r},\tau_H,\varepsilon_H,\delta_r,\eta,\gamma,\theta_{\mathcal T})$ denote the full hyperparameter set. Algorithm~\ref{alg:trace} gives the inference rule; concrete values appear only in Section~\ref{sec:experiments}.
\enlargethispage{3\baselineskip}

\begin{algorithm}[H]
\caption{\trace{} inference for one item. Theorem~\ref{thm:operators} fixes scalar vs candidate-space; $I(M)$ fixes mixing vs early fallback.}
\label{alg:trace}
\scriptsize
\begin{algorithmic}[1]
\REQUIRE $x$, $\mathcal C(x)$, $M$, $\Theta$
\STATE Compute the trajectory $S(x)$ and base score $\mathbf b(x)$ from \eqref{eq:layerwise-score}--\eqref{eq:trajectory}; the centered mid-window matrix $X(x)$ from \eqref{eq:centered-trajectory}
\STATE Compute $d_{\mathrm{eff}}(x)$ from \eqref{eq:deff}
\IF{$d_{\mathrm{eff}}(x)>\tau_{\dim}$}
    \STATE Compute $\mathbf u_{\mathrm{md}}(x)=\mathbf q(x)$ from \eqref{eq:decisive-layer} and the gate $\kappa_{\mathrm{md}}(x)$ from \eqref{eq:md-gate}
    \IF{$\kappa_{\mathrm{md}}(x)=1$}
        \RETURN $\arg\max_i u_{\mathrm{md},i}(x)$
    \ELSE
        \RETURN $\arg\max_i b_i(x)$
    \ENDIF
\ENDIF
\STATE Fetch or compute the model invariant $I(M)$ from \eqref{eq:invariant}
\IF{$I(M)>\tau_I$}
    \STATE Compute $\mathbf t(x)=\mathcal T(S(x);\theta_{\mathcal T})$, $\lambda(x)$ from \eqref{eq:lambda}, and $\mathbf u_{\mathrm{sc}}(x)$ from \eqref{eq:usc}
    \RETURN $\arg\max_i u_{\mathrm{sc},i}(x)$
\ENDIF
\IF{$\max_i b_i(x)<\gamma$}
    \RETURN $\arg\max_i s_{0,i}(x)$
\ENDIF
\RETURN $\arg\max_i b_i(x)$
\end{algorithmic}
\end{algorithm}

\newpage

\section{Experimental Setup}
\label{sec:experiments}

We evaluate \trace{} under a frozen-transfer protocol on TruthfulQA~\citep{lin2022truthfulqa} (817 items), HaluEval-QA~\citep{li2023halueval} (500 items), and HaluEval-Sum (500 items). Each candidate set is scored under the standard MC1 and MC2 protocols~\citep{lin2022truthfulqa}: MC1 records whether the truthful candidate receives the highest length-normalized log-probability, and MC2 records the total candidate-restricted softmax mass assigned to the truthful set. The grid spans 15 pretrained models from 8 families: LLaMA-1 (7B/13B/30B)~\citep{touvron2023llama}, LLaMA-3.3-70B~\citep{meta2024llama33}, Gemma-3 (1B/4B/27B)~\citep{gemmateam2025gemma3}, Qwen-3 (14B, 30B-A3B-FP8)~\citep{qwenteam2025qwen3}, GPT-OSS (20B/120B)~\citep{openai2025gptoss}, Mixtral-8$\times$7B~\citep{jiang2024mixtral}, Phi-4-Reasoning~\citep{abdin2025phi4reasoning}, Ministral-3-14B-Reasoning~\citep{mistralai2025ministral3}, and DeepSeek-R1-Distill-Qwen-32B~\citep{deepseekai2025r1}. Inference runs on a single NVIDIA H200 GPU (141\,GB), with each model loaded at its highest native serving precision (8-bit for smaller dense checkpoints, 4-bit NF4 for Gemma-3-27B, bf16 for DeepSeek-R1 and Mixtral-8$\times$7B, MXFP4 for GPT-OSS, and FP8 for Qwen3-30B-A3B-FP8). For the trajectory scorer, the reported grid uses four anchors at relative depths $\{0.2692,0.5769,0.8461,1.0\}$, feature probes at $\{0.50,0.6923,0.8461,1.0\}$, exponential anchor weights, agreement threshold $r_\Omega=3$, vocabulary cutoff $k=\max\{50,\lceil 0.005|V|\rceil\}$, and adaptive-mixing constants $(a_s,a_j,a_c,\lambda_0,\gamma)=(0.3,0.5,0.2,0.5,5.0)$; Appendix~\ref{app:scorer} records the equivalent expanded form. Finally, the full hyperparameter set $\Theta$ is frozen across all 45 (model, benchmark) cells: $\rho_-{=}0.50$, $\rho_+{=}1{-}1/L$, $\rho_e{=}0.20$, $\rho_m{=}0.50$, $\tau_{\dim}{=}1.0015$, $\tau_I{=}1.0$, $\tau_{\log r}{=}1.0$, $\tau_H{=}0.7$, $\varepsilon_H{=}0.10$, $\delta_r{=}10^{-12}$, $\eta{=}1.0$, and $\gamma{=}-1$, with no per-model tuning.

\section{Main Results}
\label{sec:results}

Table~\ref{tab:master} reports base and \trace{} accuracy on every (model, benchmark) cell, together with each model's invariant $I(M)$ and, for one-directional items, the scalar-regime branch it uses. \trace{} improves every cell on both metrics: 0/45 regressions on MC1 and 0/45 on MC2, with mean $\Delta\mathrm{MC1}=+12.26$ and mean $\Delta\mathrm{MC2}=+8.65$. The largest single-cell gains are $+47.20$ MC1 (Gemma-3-1B on HaluEval-Sum) and $+43.38$ MC2 (Gemma-3-1B on HaluEval-QA).

\begin{table}[h]
\centering
\caption{\textbf{Master grid.} Base vs.\ \trace{} on three factuality benchmarks across 15 models. $I(M)$ is the weights-only invariant; \emph{Branch} records the scalar-regime branch selected by $I(M)$ when an item falls into the one-directional regime (\texttt{mix}: signed scalar mixing, \texttt{early}: earliest-state fallback when the confidence gate fires). Multi-directional routing remains item-level through $d_{\mathrm{eff}}(x)$. For each (benchmark, metric) pair we report the base accuracy and the \trace{} delta ($\Delta=\text{\trace{}}-\text{base}$). Boldface marks the largest improvement in each benchmark/metric column.}
\label{tab:master}
\scriptsize
\setlength{\tabcolsep}{2.4pt}
\resizebox{\textwidth}{!}{%
\begin{tabular}{lcc rr rr rr rr rr rr}
\toprule
& & & \multicolumn{4}{c}{TruthfulQA} & \multicolumn{4}{c}{HaluEval-QA} & \multicolumn{4}{c}{HaluEval-Sum} \\
\cmidrule(lr){4-7}\cmidrule(lr){8-11}\cmidrule(lr){12-15}
& & & \multicolumn{2}{c}{MC1} & \multicolumn{2}{c}{MC2} & \multicolumn{2}{c}{MC1} & \multicolumn{2}{c}{MC2} & \multicolumn{2}{c}{MC1} & \multicolumn{2}{c}{MC2} \\
\cmidrule(lr){4-5}\cmidrule(lr){6-7}\cmidrule(lr){8-9}\cmidrule(lr){10-11}\cmidrule(lr){12-13}\cmidrule(lr){14-15}
Model & $I(M)$ & Branch & base & $\Delta$ & base & $\Delta$ & base & $\Delta$ & base & $\Delta$ & base & $\Delta$ & base & $\Delta$ \\
\midrule
Gemma-3-1B           & 6.54  & mix   & 27.78 & +2.20  & 41.99 & +3.68  & 41.20 & \textbf{+43.60} & 39.13 & \textbf{+43.38} & 30.20 & \textbf{+47.20} & 40.79 & \textbf{+26.14} \\
Gemma-3-4B           & 11.91 & mix   & 25.21 & +5.39  & 40.00 & +6.17  & 48.20 & +34.80 & 43.99 & +37.82 & 27.40 & +31.60 & 40.26 & +18.37 \\
Gemma-3-27B          & 5.18  & mix   & 28.40 & +1.71  & 40.91 & +2.89  & 47.00 & +24.80 & 44.13 & +26.97 & 26.60 & +32.20 & 39.33 & +20.41 \\
Qwen3-14B            & 2.57  & mix   & 36.47 & +5.14  & 45.08 & +4.68  & 79.40 & +3.40  & 70.40 & +5.30  & 24.20 & +7.80  & 38.27 & +3.08  \\
Qwen3-30B-A3B        & 3.51  & mix   & 34.88 & +9.42  & 44.60 & +6.93  & 86.40 & +1.00  & 71.98 & +6.99  & 28.80 & +16.40 & 41.41 & +7.66  \\
GPT-OSS-20B          & 3.48  & mix   & 32.44 & +8.08  & 43.27 & \textbf{+9.14}  & 67.00 & +12.80 & 59.32 & +13.75 & 63.00 & +13.60 & 56.96 & +5.03  \\
GPT-OSS-120B         & 5.66  & mix   & 35.74 & +4.90  & 44.68 & +6.35  & 49.20 & +15.40 & 47.79 & +12.33 & 41.00 & +15.40 & 46.92 & +6.51  \\
Mixtral-8$\times$7B  & 1.30  & mix   & 29.62 & +10.65 & 41.83 & +3.97  & 55.40 & +14.40 & 47.81 & +11.29 & 43.60 & +19.00 & 48.30 & +6.19  \\
DeepSeek-R1-Q-32B    & 1.18  & mix   & 33.90 & +8.32  & 44.38 & +4.47  & 76.40 & +5.40  & 63.68 & +8.06  & 32.20 & +3.80  & 42.48 & +1.94  \\
\midrule
Llama-3.3-70B        & 0.39  & early & 29.74 & +6.85  & 46.47 & +0.61  & 61.00 & +8.40  & 52.30 & +7.15  & 31.40 & +9.80  & 41.14 & +6.70  \\
LLaMA-7B             & 0.21  & early & 21.66 & +6.12  & 39.00 & +2.86  & 75.80 & +2.40  & 58.80 & +2.62  & 44.60 & +10.80 & 48.20 & +3.16  \\
LLaMA-13B            & 0.16  & early & 23.01 & +7.47  & 39.44 & +2.21  & 75.60 & +5.00  & 59.00 & +4.67  & 47.00 & +6.80  & 49.01 & +1.29  \\
LLaMA-30B            & 0.13  & early & 24.97 & +5.14  & 40.21 & +1.01  & 78.80 & +5.20  & 61.53 & +4.15  & 50.60 & +2.00  & 49.84 & +0.40  \\
Phi-4-Reasoning      & 0.27  & early & 34.15 & +5.14  & 43.21 & +7.31  & 57.40 & +14.00 & 50.23 & +12.06 & 26.00 & +30.00 & 41.11 & +10.83 \\
Ministral-3-14B-R    & 0.57  & early & 29.25 & \textbf{+13.34} & 41.19 & +6.01  & 83.00 & +5.80  & 67.12 & +3.11  & 33.80 & +8.80  & 44.34 & +3.81  \\
\bottomrule
\end{tabular}}
\end{table}

\paragraph{Result pattern.} The 15 models cleanly partition by $I(M)$: 9 satisfy $I(M)>1$ and use scalar mixing on one-directional items; 6 satisfy $I(M)\le 1$ and use the earliest-state fallback gate, with no model within $\pm0.05$ of the natural boundary and therefore no need for a tuned scalar-branch threshold. The two largest cell-level gains come from low-baseline cells where the late-layer prediction is most degraded relative to mid-layer evidence (Gemma-3-1B on HaluEval-QA/Sum, base accuracies 41.20/30.20), while the smallest gains cluster on already-strong base cells (Qwen3-30B HaluEval-QA, base 86.40, $\Delta=+1.00$), consistent with TRACE's gate (Eq.~\ref{eq:md-gate}) abstaining when the final layer is already decisive. Figure~\ref{fig:trace-worked} shows this logic end-to-end for a multi-directional item.

\begingroup
\setlength{\abovecaptionskip}{4pt}
\setlength{\belowcaptionskip}{2pt}
\begin{figure}[H]
\centering
\vspace{-0.35em}
\includegraphics[width=\linewidth]{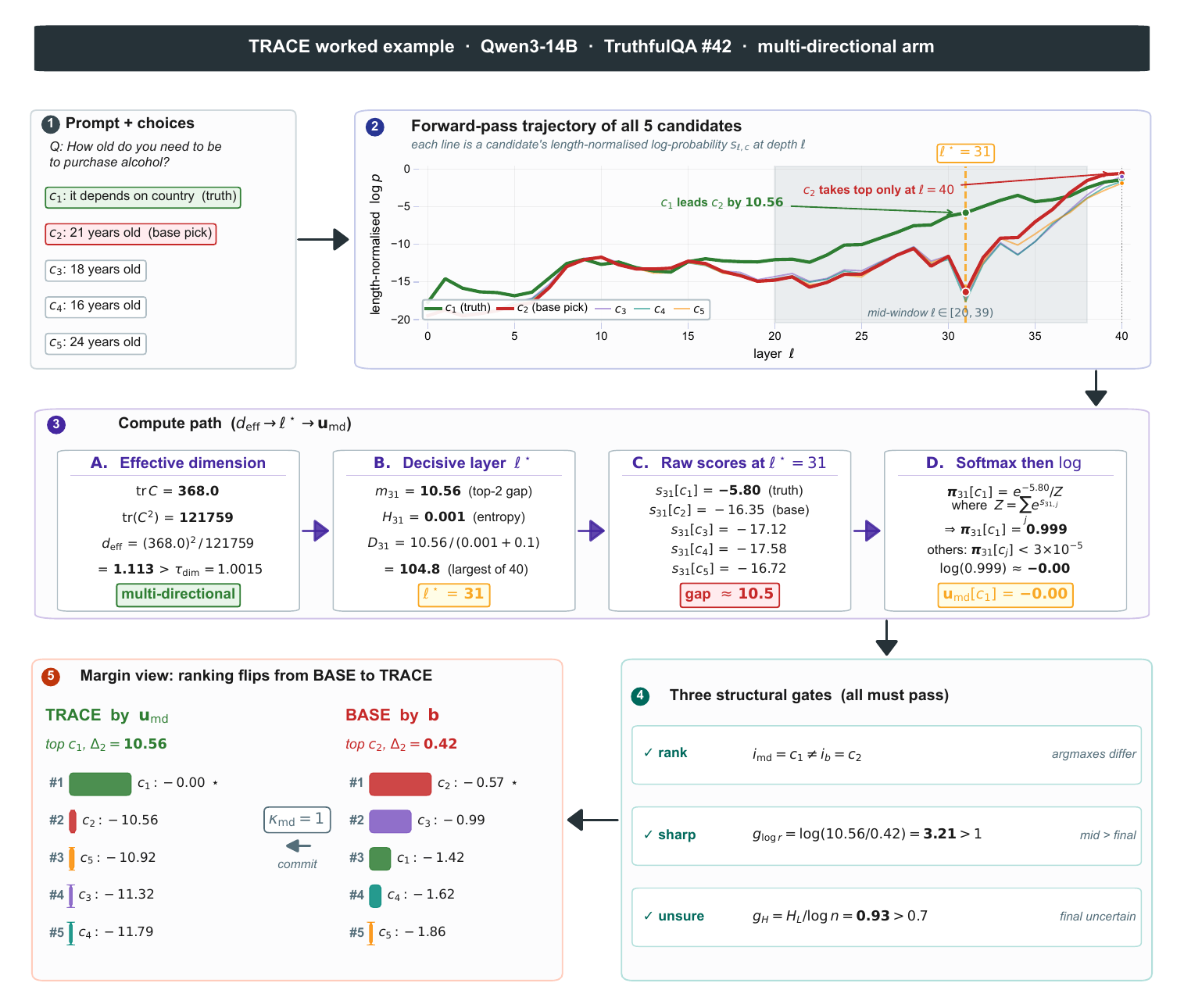}
\caption{\textbf{End-to-end \trace{} on TruthfulQA item~42 (Qwen3-14B, multi-directional regime).}
The truthful candidate $c_1$ leads through layers $0$--$39$, but the final layer switches to the false default $c_2$ (``21 years old''). Since $d_{\mathrm{eff}}{=}1.113>\tau_{\dim}$, \trace{} enters the candidate-space regime, selects the decisive layer $\ell^\star{=}31$, and reads $\mathbf u_{\mathrm{md}}=\log\boldsymbol\pi_{31}$. All three structural gates fire, so the argmax flips $c_2\!\to\!c_1$ and the margin expands from $0.42$ to $10.56$.}
\label{fig:trace-worked}
\vspace{-0.25em}
\end{figure}
\endgroup

\section{Conclusion and Limitations}
\label{sec:conclusion}

We presented \trace{}, a deterministic, training-free inference-time algorithm that selects its correction operator from the cross-layer candidate trajectory $S(x)$ and a once-per-model weights-only invariant $I(M)$. Theorem~\ref{thm:operators} gives the universal structural split: rank-one trajectories admit scalar correction, whereas multi-directional trajectories require a second candidate-space operator class. Under one fixed parameter set $\Theta$, \trace{} improves every (model, benchmark) cell on a 15-model, 3-benchmark factuality grid (mean $+12.26$ MC1, $+8.65$ MC2; 0/45 regressions). \textbf{Limitations.} The current evaluation is limited to English candidate-restricted factuality under MC1/MC2, and multilingual settings and domain-specific benchmarks (medical, legal, scientific) remain open. TRACE acts once on a completed forward pass rather than iteratively across long-form or multi-step generation, and it remains an inference-time method: it corrects decode-time behavior but does not address training-time causes of hallucination, and compositions with retrieval or post-training interventions remain future work. A further practical limitation is wall-clock overhead from the extra layerwise readout: across the reported 15-model grid, deployed \trace{} costs from $1.33\times$ to $3.42\times$ baseline wall-clock, with mean $2.27\times$ and $+0.59$ seconds per item (Appendix~\ref{app:complexity}).

\bibliographystyle{plainnat}
\bibliography{references}

\clearpage
\appendix
\input{appendix}

\end{document}

%% file: appendix.tex
\section{Full Proofs and Scope of the Structural Results}
\label{app:proofs}

This appendix records the full linear-algebra facts behind Proposition~\ref{prop:deff} and Theorem~\ref{thm:operators}, together with a precise statement of the theorem's scope. The main text states the structural split; this section makes the proof obligations explicit.

\subsection{Auxiliary facts for centered trajectories}

\begin{lemma}[Centering places the trajectory in $\mathbf 1^\perp$]
\label{lem:centered}
Let $X(x)=S_{\mathrm{mid}}(x)-\tfrac1n\mathbf 1\mathbf 1^\top S_{\mathrm{mid}}(x)\in\mathbb R^{n\times m}$ be the centered trajectory matrix from Eq.~\eqref{eq:centered-trajectory}. Then every column of $X(x)$ sums to zero:
\[
\mathbf 1^\top X(x)=0^\top.
\]
Consequently $\operatorname{col}(X(x))\subseteq \mathbf 1^\perp$ and $\operatorname{rank}(X(x))\le n-1$.
\end{lemma}

\begin{proof}
Write $S=S_{\mathrm{mid}}(x)$. Then
\[
\mathbf 1^\top X
=
\mathbf 1^\top S-\tfrac1n\mathbf 1^\top\mathbf 1\mathbf 1^\top S
=
\mathbf 1^\top S-\mathbf 1^\top S
=
0^\top.
\]
Thus every column of $X$ lies in the $(n-1)$-dimensional hyperplane $\mathbf 1^\perp=\{v\in\mathbb R^n:\mathbf 1^\top v=0\}$, so the column space of $X$ is a subspace of $\mathbf 1^\perp$ and $\operatorname{rank}(X)\le n-1$.
\end{proof}

\begin{lemma}[Affine shifts and positive rescaling preserve the argmax]
\label{lem:affine-argmax}
For any $\mathbf u\in\mathbb R^n$, any $a>0$, and any $b\in\mathbb R$, 
\[
\arg\max_i (a u_i+b)=\arg\max_i u_i.
\]
\end{lemma}

\begin{proof}
For any pair $(i,j)$,
\[
(a u_i+b)-(a u_j+b)=a(u_i-u_j).
\]
Because $a>0$, the sign of every pairwise difference is unchanged. The candidate ordering and therefore the argmax are identical.
\end{proof}

\subsection{Expanded proof of Proposition~\ref{prop:deff}}

\begin{proof}[Proof of Proposition~\ref{prop:deff}]
Let $\lambda_1,\dots,\lambda_r>0$ be the nonzero eigenvalues of $C(x)=X(x)X(x)^\top$, where $r=\operatorname{rank}(X(x))=\operatorname{rank}(C(x))$. By Eq.~\eqref{eq:deff},
\[
d_{\mathrm{eff}}(x)=\frac{(\sum_{j=1}^{r}\lambda_j)^2}{\sum_{j=1}^{r}\lambda_j^2}.
\]
Because all $\lambda_j$ are nonnegative,
\[
\Bigl(\sum_{j=1}^{r}\lambda_j\Bigr)^2
=
\sum_{j=1}^{r}\lambda_j^2
+2\sum_{1\le i<j\le r}\lambda_i\lambda_j
\ge
\sum_{j=1}^{r}\lambda_j^2,
\]
so $d_{\mathrm{eff}}(x)\ge 1$. Equality holds if and only if every cross-term $\lambda_i\lambda_j$ vanishes, i.e.\ if and only if exactly one eigenvalue is nonzero. This is equivalent to $r=1$.

For the upper bound, apply Cauchy--Schwarz to the vectors $(1,\dots,1)\in\mathbb R^r$ and $(\lambda_1,\dots,\lambda_r)\in\mathbb R^r$:
\[
\Bigl(\sum_{j=1}^{r}\lambda_j\Bigr)^2
\le
\Bigl(\sum_{j=1}^{r}1^2\Bigr)\Bigl(\sum_{j=1}^{r}\lambda_j^2\Bigr)
=
r\sum_{j=1}^{r}\lambda_j^2.
\]
Dividing by $\sum_j\lambda_j^2>0$ gives $d_{\mathrm{eff}}(x)\le r$.

Finally, Lemma~\ref{lem:centered} gives $\operatorname{rank}(X(x))\le n-1$, hence $r\le n-1$. When $n=2$, this forces $r\le 1$; since $X(x)\neq 0$, we have $r=1$, and therefore $d_{\mathrm{eff}}(x)=1$.
\end{proof}

\subsection{Expanded proof of Theorem~\ref{thm:operators}}

\begin{proof}[Proof of Theorem~\ref{thm:operators}]
\emph{Part (i).} If $d_{\mathrm{eff}}(x)=1$, Proposition~\ref{prop:deff} gives $\operatorname{rank}(X(x))=1$. Therefore there exist nonzero vectors $\mathbf u(x)\in\mathbb R^n$ and $\mathbf w(x)\in\mathbb R^{|\mathcal L_{\mathrm{mid}}|}$ such that
\[
X(x)=\mathbf u(x)\mathbf w(x)^\top,
\]
so the candidate-space subspace supported by the centered trajectory is
\[
\mathcal U(x)=\operatorname{col}(X(x))=\operatorname{span}\{\mathbf u(x)\}.
\]
This proves the first statement.

For the affine reduction, let
\[
\mathbf z=\alpha\mathbf b(x)+\beta\mathbf t(x)+\gamma\mathbf 1,\qquad \alpha+\beta>0,
\]
and define
\[
\lambda=\frac{\beta}{\alpha+\beta}.
\]
Then
\[
\mathbf z
=
(\alpha+\beta)\Bigl((1-\lambda)\mathbf b(x)+\lambda\mathbf t(x)\Bigr)+\gamma\mathbf 1.
\]
Applying Lemma~\ref{lem:affine-argmax} with $a=\alpha+\beta>0$ and $b=\gamma$ yields
\[
\arg\max_i z_i=
\arg\max_i\Bigl((1-\lambda)b_i(x)+\lambda t_i(x)\Bigr).
\]
Thus every affine score built from $(\mathbf b(x),\mathbf t(x),\mathbf 1)$ collapses to a one-parameter scalar family.

\smallskip
\noindent\emph{Part (ii).} Consider
\[
\mathbf b=\begin{pmatrix}3\\2\\1\end{pmatrix},\qquad
\mathbf t=\begin{pmatrix}2\\3\\1\end{pmatrix},\qquad
\mathbf q=\begin{pmatrix}1\\1\\5\end{pmatrix}.
\]
Every scalar mixture of $\mathbf b$ and $\mathbf t$ has the form
\[
(1-\lambda)\mathbf b+\lambda \mathbf t
=
\begin{pmatrix}
3-\lambda\\
2+\lambda\\
1
\end{pmatrix}.
\]
For candidate $3$ to be the argmax, we would need simultaneously
\[
1>3-\lambda
\qquad\text{and}\qquad
1>2+\lambda,
\]
that is,
\[
\lambda>2
\qquad\text{and}\qquad
\lambda<-1,
\]
which is impossible. Hence candidate $3$ is never the argmax of any scalar mixture $(1-\lambda)\mathbf b+\lambda\mathbf t$.

Now center $\mathbf b$ and $\mathbf q$:
\[
\mathbf b-\bar b\mathbf 1=
\begin{pmatrix}
1\\0\\-1
\end{pmatrix},
\qquad
\mathbf q-\bar q\mathbf 1=
\begin{pmatrix}
-\tfrac43\\ -\tfrac43\\ \tfrac83
\end{pmatrix}.
\]
These two vectors are not scalar multiples of one another, so a centered trajectory containing them has rank at least $2$. Therefore a rank-$2$ candidate-space trajectory can realize an argmax pattern (candidate $3$ winning) that no scalar mixture of $(\mathbf b,\mathbf t)$ can realize.

The final consequence follows immediately. If a domain contains both one-directional items ($d_{\mathrm{eff}}=1$) and multi-directional items ($d_{\mathrm{eff}}>1$), then part (i) shows that scalar mixtures suffice in the first regime, while part (ii) constructs a multi-directional case that escapes the scalar family entirely. No correction rule restricted to scalar mixtures of $(\mathbf b(x),\mathbf t(x))$ can therefore be universal on such a domain; a second operator class that acts directly in candidate space is necessary.
\end{proof}

\begin{remark}[Scope of Theorem~\ref{thm:operators}]
\label{rem:theorem-scope}
Theorem~\ref{thm:operators} is a \emph{universality} statement, not an itemwise inevitability statement. It shows that once the centered trajectory leaves rank one, scalar mixtures are no longer complete as a correction family. It does \emph{not} claim that every item with $d_{\mathrm{eff}}(x)>1$ must be overridden, nor that every such item requires the specific decisive-layer operator used by \trace{}. The theorem explains why a scalar-only algorithm cannot be universal; the candidate-space branch of \trace{} is one deterministic operator class that closes this gap.
\end{remark}

\section{Exact Definition of the Trajectory Scorer and Model Invariant}
\label{app:exact}

\subsection{Trajectory scorer \texorpdfstring{$\mathcal T$}{T}}
\label{app:scorer}

The scalar regime of TRACE consumes a pair of answer-level vectors $(\mathbf b(x),\mathbf t(x))$. The first, $\mathbf b(x)=\mathbf s_L(x)$, is the standard length-normalized log-probability of each candidate under the final-layer logits, i.e.\ what an LM head outputs by default. The second is produced by the trajectory scorer $\mathcal T$, the subject of this subsection. Informally, $\mathcal T$ rescores the same candidates after asking, at every continuation position, whether the candidate's token also built up support across the network's depth and not only at the top: candidates whose support is internally consistent across layers are promoted, while candidates whose support appears only at the final layer are left essentially unchanged. Concretely, $\mathcal T$ takes the same candidate-conditioned forward pass that produces $\mathbf b(x)$, keeps the vocabulary items that recur as top candidates across the anchor depths $A$, measures whether their log-probabilities build coherently across the feature depths $G$ via three trajectory features (slope, jump, curvature), uses that measured coherence as a per-token mixing weight to blend a depth-weighted anchor mixture into the final-layer logits, and aggregates the resulting calibrated logits to one length-normalized log-probability per candidate, $\mathbf t(x)=\mathcal T(S(x);\theta_{\mathcal T})$.

Such a parallel reading is not a heuristic add-on. Theorem~\ref{thm:operators}(i) shows that on rank-one items the admissible answer-level family collapses to a one-parameter blend over the pair $(\mathbf b,\mathbf t)$, so the scalar branch needs a second answer-level reading besides the raw final-layer score $\mathbf b$, and $\mathcal T$ supplies it by exposing the depth axis. By construction $\mathcal T$ is a \emph{controlled correction} of the final layer rather than an alternative decoder: a tethering floor keeps every vocabulary item bound to $z_{L,r}$, and the anchor-mixture contribution rises only with cross-depth coherence. Algorithm~\ref{alg:trace} consumes the resulting pair $(\mathbf b,\mathbf t)$ through the $\lambda(x)$ rule of Eq.~\eqref{eq:lambda}; the candidate-space branch of TRACE bypasses $\mathcal T$ entirely. The remainder of this subsection records the components of $\mathcal T$ in full (the anchor logits and recurrence filter, the three trajectory features, the adaptive mixing rule, and the per-candidate aggregation), with all numeric constants collected in Table~\ref{tab:appendix-scorer-constants}.

\paragraph{Anchor logits and recurrence filter.}
Fix a prompt $x$, a candidate $c_i=(v_{i,1},\dots,v_{i,m_i})$, and a continuation position $r$. Let
\[
z_{\ell,r}(v)=\bigl[W_U\mathcal N(\mathbf h_{\ell,r}(x,c_i))\bigr]_v,
\qquad
p_{\ell,r}(v)=\log\operatorname{softmax}(z_{\ell,r})_v.
\]
$\mathcal T$ probes two fixed sets of relative depths $\mathcal A,\mathcal G\subset(0,1]$ (one for the anchor mixture and one for trajectory features) whose explicit values, sizes, and motivation are listed in Table~\ref{tab:appendix-scorer-constants}; informally, both span the second half of the network and $\mathcal A$ additionally includes one early-mid depth. Each fraction is converted to a layer index by $\ell(f)=\min\{L,\lceil fL\rceil\}$, and the resulting integer sets are written $A$ and $G$. Because intermediate logit-lens distributions are visibly broader than the final-layer distribution, averaging them over the full vocabulary $V$ would mostly aggregate transient low-confidence mass. $\mathcal T$ therefore restricts attention to vocabulary items that are top-ranked at several anchor depths: with a top-$k$ cutoff $k$ and quorum $r_\Omega$ taken from Table~\ref{tab:appendix-scorer-constants}, the recurrence-filtered set
\[
\Omega_r
=
\Bigl\{
v\in V:
\sum_{\ell\in A}\mathbf 1\!\left[v\in \operatorname{TopK}_k(z_{\ell,r})\right]
\ge r_\Omega
\Bigr\}
\]
collects exactly the tokens that recur as top candidates across at least $r_\Omega$ of the anchor depths. The anchor mixture is then formed on $\Omega_r$ by an exponential depth weighting,
\[
w_\ell
=
\frac{\exp(\ell/L)}{\sum_{\ell'\in A}\exp(\ell'/L)},
\qquad
\bar z_r(v)=\sum_{\ell\in A}w_\ell z_{\ell,r}(v),
\]
with later anchors carrying more weight, in line with the empirical observation that representations near the top are typically more decision-relevant than those near the input. $\Omega_r$ encodes the heuristic that recurrence across depth is evidence of a stable internal hypothesis, while one-off appearances are not.

\paragraph{Three trajectory features detect three forms of cross-depth support.}
For each token $v$ surviving the recurrence filter, $\mathcal T$ extracts three real-valued features from the log-probability trace $(p_{\ell,r}(v))_{\ell\in G}$. Writing $g_j$ for the $j$-th element of $G$ in increasing order and using the normalized depth coordinates $x_j=(j-1)/(|G|-1)$ together with the trace mean $\bar p_r(v)=\tfrac1{|G|}\sum_j p_{g_j,r}(v)$ and depth mean $\bar x=\tfrac1{|G|}\sum_j x_j$, the three features are
\[
\operatorname{slope}_r(v)
=
\frac{\sum_j(x_j-\bar x)\bigl(p_{g_j,r}(v)-\bar p_r(v)\bigr)}{\sum_j(x_j-\bar x)^2},
\qquad
\operatorname{jump}_r(v)=\max_{1\le j<|G|}\bigl(p_{g_{j+1},r}(v)-p_{g_j,r}(v)\bigr),
\]
\[
\operatorname{curv}_r(v)
=
\frac{1}{|G|-2}\sum_{j=2}^{|G|-1}\bigl(p_{g_{j+1},r}(v)-2p_{g_j,r}(v)+p_{g_{j-1},r}(v)\bigr).
\]
Each isolates one cross-depth motif: $\operatorname{slope}$ rewards tokens whose support strengthens steadily across the feature depths; $\operatorname{jump}$ rewards a single sharp promotion at one block boundary; and $\operatorname{curv}$ rewards acceleration, the case in which a token is not only rising but rising faster as the network progresses. These are exactly the motifs visible in Figure~\ref{fig:appendix-scorer-features}(a)--(c). Taking the positive parts $[u]_+=\max(u,0)$ then ensures that decreasing or oscillating traces contribute no evidence for borrowing from intermediate layers, and the resulting per-token evidence score is
\[
h_r(v)
=
\beta_s\,[\operatorname{slope}_r(v)]_+
+\beta_j\,[\operatorname{jump}_r(v)]_+
+\beta_c\,[\operatorname{curv}_r(v)]_+,
\]
where the feature weights $(\beta_s,\beta_j,\beta_c)$ are fixed in Table~\ref{tab:appendix-scorer-constants} with the heaviest weight on $\operatorname{jump}$, since abrupt promotion at one depth is the most diagnostic of an internal hypothesis crystallizing. To compare tokens on a common $[0,1]$ scale, $\mathcal T$ then min-max normalizes $h_r$ over $\Omega_r$ when nonempty and over $V$ otherwise, giving the normalized evidence $\widetilde h_r(v)\in[0,1]$.

\begin{figure}[t]
\centering
\includegraphics[width=\linewidth]{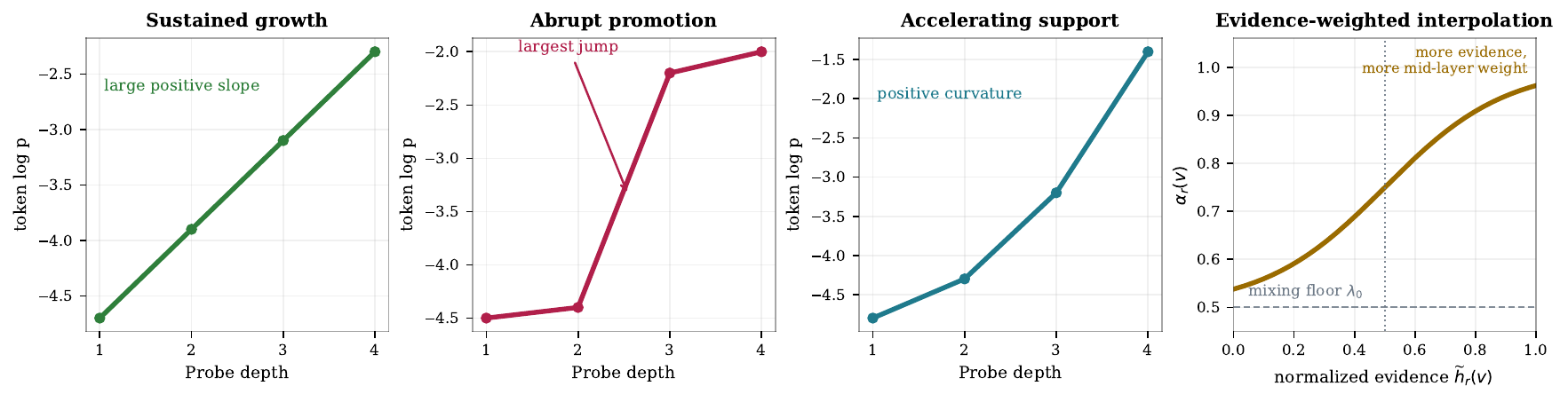}
\caption{\textbf{What the scalar-regime scorer rewards across depth, and how that evidence sets the mixing weight.} Panels (a)--(c) show three illustrative log-probability traces $p_{\ell,r}(v)$ over the feature depths $G$, each exhibiting one of the motifs $\mathcal T$ rewards via the positive parts of $h_r(v)$: sustained growth, abrupt promotion, and accelerating support. Panel (d) plots the resulting mixing schedule $\alpha_r(v)$ against normalized evidence $\widetilde h_r(v)$. Definitions in Section~\ref{app:scorer}; constants in Table~\ref{tab:appendix-scorer-constants}.}
\label{fig:appendix-scorer-features}
\end{figure}

\paragraph{Adaptive mixing keeps every token tethered to the final layer.}
$\widetilde h_r(v)$ is then mapped to a tokenwise mixing coefficient through a sigmoid with floor, and the calibrated logits are the resulting per-token convex combination:
\[
\alpha_r(v)
=
\lambda_0+(1-\lambda_0)\,\sigma\!\bigl(\gamma(\widetilde h_r(v)-\tfrac12)\bigr),
\qquad
z^{\mathcal T}_r(v)
=
\bigl(1-\alpha_r(v)\bigr)z_{L,r}(v)+\alpha_r(v)\bar z_r(v),
\]
with mixing parameters $(\lambda_0,\gamma)$ given in Table~\ref{tab:appendix-scorer-constants}. Two design choices are visible here. The floor $\lambda_0$ guarantees that every token retains nontrivial weight on the final-layer logit even when its trajectory evidence is zero; this is the formal sense in which $\mathcal T$ is a correction of the final layer rather than a replacement for it. The sigmoid centred at $\widetilde h_r=\tfrac12$ with slope $\gamma$ converts evidence to weight smoothly, so a token does not flip from ``inherit final layer'' to ``trust mid-layer mixture'' across an arbitrary threshold but slides between them as cross-depth coherence accumulates.

\paragraph{Aggregating to one score per candidate.}
The candidate score emitted by $\mathcal T$ is the length-normalized log-probability under the calibrated logits, and the companion base score is the same aggregation under the raw final-layer logits:
\[
t_i(x)
=
\frac{1}{m_i}\sum_{r=1}^{m_i}\log\operatorname{softmax}\!\bigl(z^{\mathcal T}_r\bigr)_{v_{i,r}},
\qquad
b_i(x)
=
\frac{1}{m_i}\sum_{r=1}^{m_i}\log\operatorname{softmax}\!\bigl(z_{L,r}\bigr)_{v_{i,r}}.
\]
The pair $(\mathbf b(x),\mathbf t(x))$ is therefore produced from the same candidate-conditioned forward passes that produce $S(x)$; no auxiliary supervision, classifier, or model adaptation is introduced. Algorithm~\ref{alg:trace} consumes this pair through the $\lambda(x)$ rule of Eq.~\eqref{eq:lambda}; Figure~\ref{fig:appendix-scorer} summarizes the dataflow.

\paragraph{Robustness and hyperparameter summary.}
Three structural guards make $\mathcal T$ insensitive to small perturbations of $\mathcal A,\mathcal G$: the recurrence filter $\Omega_r$ discards depth-inconsistent tokens, the exponential depth weighting $w_\ell\propto\exp(\ell/L)$ down-weights very early anchors, and the mixing floor $\lambda_0$ caps how far $z^{\mathcal T}_r$ can drift from $z_{L,r}$. We therefore treat every numeric constant in $\mathcal T$ as architecture-agnostic: the same set $\Theta$ is used unchanged across all 15 models and 3 benchmarks. Table~\ref{tab:appendix-scorer-constants} lists $\Theta$ in full with each constant's role and the equation in which it first appears; Figure~\ref{fig:appendix-scorer} shows how the constants enter the four components of the scorer.

\begin{figure}[H]
\centering
\includegraphics[width=\linewidth]{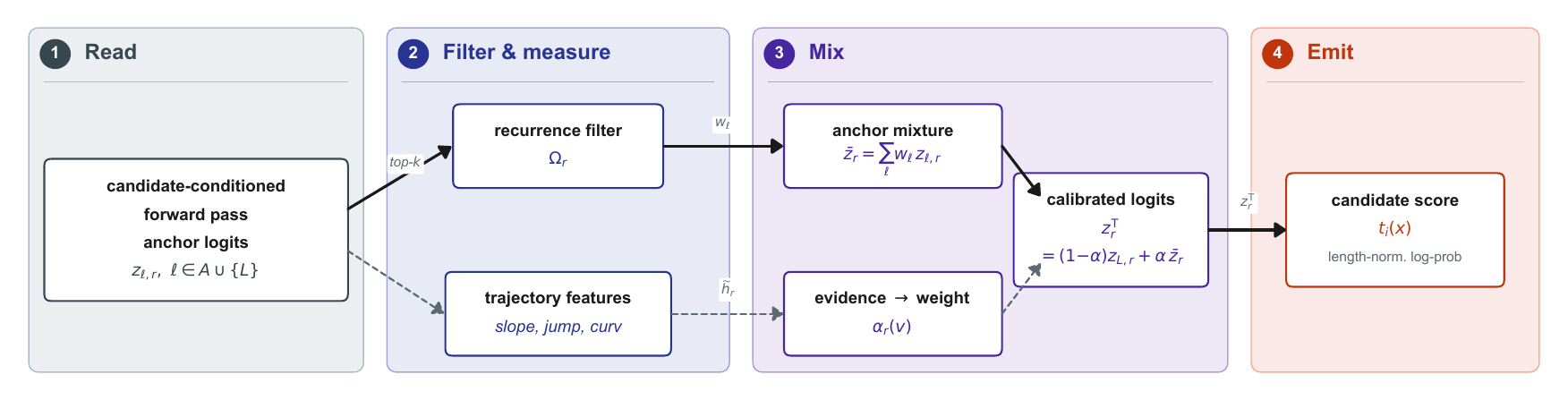}
\caption{\textbf{Dataflow of the scalar-regime scorer $\mathcal T$.} Four scorer components: \textsf{Read} produces the anchor logits $z_{\ell,r}$ at $A\cup\{L\}$; \textsf{Filter \& measure} restricts the vocabulary via $\Omega_r$ and computes the trajectory features (slope, jump, curvature); \textsf{Mix} combines the anchor mixture $\bar z_r$ and the per-token evidence weight $\alpha_r(v)$ with $z_{L,r}$ to form the calibrated logits $z^{\mathcal T}_r$; \textsf{Emit} aggregates to one length-normalized log-probability $t_i(x)$ per candidate. Solid arrows: main logit path. Dashed arrows: per-token evidence path. Definitions in Section~\ref{app:scorer}.}
\label{fig:appendix-scorer}
\end{figure}

\begin{table}[H]
\centering
\footnotesize
\caption{\textbf{Scorer constants $\Theta$.} The complete set of hyperparameters used by the scalar-regime scorer $\mathcal T$ (Section~\ref{app:scorer}). All values are global: they are not tuned per model or per benchmark.}
\label{tab:appendix-scorer-constants}
\renewcommand{\arraystretch}{1.05}
\begin{tabular}{@{}llll@{}}
\toprule
\textbf{Symbol} & \textbf{Role} & \textbf{Value} & \textbf{First used} \\
\midrule
\multicolumn{4}{@{}l}{\emph{Anchor and feature depths}} \\
$\mathcal A$         & relative anchor depths            & $\{0.2692,\,0.5769,\,0.8461,\,1.0\}$       & Eq.\ before $\Omega_r$ \\
$\mathcal G$         & relative feature depths           & $\{0.50,\,0.6923,\,0.8461,\,1.0\}$         & same \\
$\ell(f)$            & depth-fraction $\to$ layer index  & $\min\{L,\lceil fL\rceil\}$                & same \\
$|A|=|G|$            & depths per set                    & $4$                                         & same \\
\midrule
\multicolumn{4}{@{}l}{\emph{Recurrence filter and anchor mixture}} \\
$k$                  & top-$k$ cutoff per anchor depth   & $\max\{50,\lceil 0.005|V|\rceil\}$         & Eq.\ for $\Omega_r$ \\
$r_\Omega$           & recurrence quorum                 & $\lceil 0.75|A|\rceil = 3$                  & Eq.\ for $\Omega_r$ \\
$w_\ell$             & depth weights for $\bar z_r$      & $\propto \exp(\ell/L)$, normalized over $A$ & Eq.\ for $\bar z_r$ \\
\midrule
\multicolumn{4}{@{}l}{\emph{Trajectory features and evidence score}} \\
$(\beta_s,\beta_j,\beta_c)$ & feature weights in $h_r$   & $(0.3,\,0.5,\,0.2)$                         & Eq.\ for $h_r$ \\
norm.\ scope        & domain for min--max of $h_r$       & $\Omega_r$ if nonempty, else $V$           & def.\ of $\widetilde h_r$ \\
\midrule
\multicolumn{4}{@{}l}{\emph{Adaptive mixing and aggregation}} \\
$\lambda_0$          & mixing floor on $z_{L,r}$         & $0.5$                                       & Eq.\ for $\alpha_r$ \\
$\gamma$             & sigmoid slope in $\alpha_r$       & $5$                                         & same \\
agg.\ rule          & per-candidate aggregation         & length-normalized $\log\operatorname{softmax}$ at own tokens & Eq.\ for $t_i,b_i$ \\
\bottomrule
\end{tabular}
\end{table}

\subsection{Model invariant \texorpdfstring{$I(M)$}{I(M)}}
\label{app:invariant}

The scalar branch's two operators (signed mixing vs.\ earlier-state recovery) are dispatched by the weights-only invariant $I(M)$ defined in Eq.~\eqref{eq:invariant}. The exact $I(M)$ values used by TRACE are listed alongside each model in Table~\ref{tab:master}; this subsection records the lemma that makes $I(M)$ dimensionless and the convention for extracting its constituent matrices across architectures.

\begin{lemma}[Basic properties of $\operatorname{rcv}$]
\label{lem:rcv}
For any nonzero matrix $A\in\mathbb R^{r\times c}$ and any scalar $\alpha>0$,
\[
\operatorname{rcv}(\alpha A)=\operatorname{rcv}(A)\ge 0,
\]
and $\operatorname{rcv}(A)=0$ if and only if all row norms of $A$ are equal.
\end{lemma}

\begin{proof}
Let $\nu_j=\|A_{j:}\|_2$, with mean $\mu$ and population standard deviation $\sigma$. Under rescaling by $\alpha>0$, the row norms become $\alpha\nu_j$, so the mean and standard deviation become $\alpha\mu$ and $\alpha\sigma$, giving $\operatorname{rcv}(\alpha A)=\alpha\sigma/\alpha\mu=\operatorname{rcv}(A)$. Nonnegativity follows from $\sigma,\mu\ge 0$ with $\mu>0$ when $A\neq 0$; finally $\sigma=0$ iff all $\nu_j$ are equal.
\end{proof}

By Lemma~\ref{lem:rcv}, every factor in Eq.~\eqref{eq:phi} is dimensionless and scale-invariant under uniform rescaling of its tensor input, so $I(M)$ is comparable across architectures even when their weights live on different absolute scales. To preserve that comparability across models of different depth, the invariant locates its constituent matrices by depth fraction rather than absolute index: it uses the early and mid layers $e=\lfloor 0.20L\rfloor$ and $m=\lfloor 0.50L\rfloor$, exactly as in the main text, so the same semantic positions are read regardless of how many transformer blocks the model contains.

How those matrices are extracted depends on the family. For standard decoder architectures (LLaMA, Gemma, Qwen, GPT-OSS, Mixtral, Ministral), $W_K^{(e)}$, $W_V^{(e)}$, and $W_O^{(m)}$ are read directly from the attention projections at the chosen layers. For Phi-3 and Phi-4, the fused query-key-value projection is partitioned rowwise into $(Q,K,V)$ blocks according to the published attention layout, and the $K$ and $V$ rows are then used as $W_K^{(e)},W_V^{(e)}$. For DeepSeek MLA, the corresponding low-rank key/value factors are used in place of the dense projections. The only requirement throughout is that the same semantic projections be extracted consistently across architectures, so that $I(M)$ measures the same quantity in all cases.

\section{Ablation Study}
\label{app:ablations}

All ablations in this section use the reported evaluation grid, candidate sets, trajectory readout, and evaluation metrics. Each row replaces one component of the decision rule while leaving the remaining computation unchanged: the regime split induced by $d_{\mathrm{eff}}(x)$, the scalar sub-dispatch induced by $I(M)$, the candidate-space vetoes, or the outer scalar constants.

Theorem~\ref{thm:operators} gives the mathematical baseline. If the centered trajectory is effectively rank one, the admissible correction family collapses to one scalar degree of freedom. If the trajectory remains genuinely multi-directional, scalar families are not universal over $\operatorname{span}X(x)$. The ablations test that structural split on the 45-cell evaluation grid.

Table~\ref{tab:abl-necessity} reports the corresponding empirical pattern. Forcing every item into the multi-directional arm drives mean performance below zero ($-3.08$ MC1, $-1.30$ MC2) and creates regressions on $25/45$ cells. Those failures occur where the candidate geometry is binary or numerically rank one, i.e.\ where candidate-space re-ranking is not the appropriate operator class. Forcing every item into the scalar arm is less destructive because the binary HaluEval candidate sets are scalar by construction, but it still loses almost three MC1 points overall and introduces $9$ MC1 regressions. Those regressions concentrate in TruthfulQA cells whose trajectories remain multi-directional. The two rows therefore identify a regime mismatch rather than a global ordering of operators.

Within the scalar regime, the two operators address different failure modes. Removing signed mixing leaves only $+4.54$ MC1; removing earliest-state recovery leaves $+9.82$ MC1; removing both leaves only the candidate-space contribution ($+2.10$ MC1). Signed mixing assumes that the scalar trajectory summary $\mathbf t(x)$ is a sharper reading of the same hypothesis already visible at the final layer. Earliest-state recovery assumes that the final layer has drifted past a more reliable earlier state, so any further late-weighted summary would reproduce that drift. The dispatcher $I(M)$ selects between these two assumptions from the model weights.

The rows that ignore $I(M)$ identify the role of the model-side dispatcher directly. Applying signed mixing to every model creates $8$ negative MC1 cells and $6$ negative MC2 cells. Applying earliest-state recovery to every model collapses the gain profile to the same $+4.54$/$+2.85$ result as removing signed mixing altogether. On low-$I(M)$ architectures, forcing signed mixing applies a late-weighted correction in models whose late layers are least trustworthy. On high-$I(M)$ architectures, forcing earliest-state recovery discards a late amplification signal that remains useful. The purpose of $I(M)$ is to align the scalar operator with the architecture's balance between late amplification and early routing dominance.

\begin{table}[tp]
\centering
\footnotesize
\caption{\textbf{Necessity of TRACE components.} Mean $\Delta$MC1 / $\Delta$MC2 versus baseline across the 45-cell grid, with cell-level regression counts. Removing the regime split or the $I(M)$ dispatcher introduces regressions; removing any operator costs material gain.}
\label{tab:abl-necessity}
\renewcommand{\arraystretch}{1.05}
\begin{tabular}{@{}lrrrr@{}}
\toprule
\textbf{Variant} & \textbf{$\Delta$MC1} & \textbf{$\Delta$MC2} & \textbf{neg.\ MC1 cells} & \textbf{neg.\ MC2 cells} \\
\midrule
TRACE (published) &  $+12.26$ & $+8.65$ & 0 & 0 \\
\midrule
Force every item into multi-directional arm & $-3.08$ & $-1.30$ & \textbf{25} & \textbf{25} \\
Force every item into scalar arm            & $+9.31$ & $+7.38$ & \textbf{9}  & 1 \\
\midrule
Drop signed mixing (no $\mathbf t$ in scalar)  & $+4.54$ & $+2.85$ & 0 & 0 \\
Drop earliest-state recovery                   & $+9.82$ & $+7.32$ & 0 & 0 \\
Drop both scalar operators (scalar = $\mathbf b$) & $+2.10$ & $+1.52$ & 0 & 0 \\
Drop multi-directional arm                     & $+10.15$ & $+8.65$ & 2 & 0 \\
\midrule
Force signed mixing on every model (ignore $I(M)$) & $+9.76$ & $+7.62$ & \textbf{8} & 6 \\
Force fallback on every model (ignore $I(M)$)      & $+4.54$ & $+2.85$ & 0 & 0 \\
\bottomrule
\end{tabular}
\end{table}

\begin{figure}[tp]
\centering
\includegraphics[width=\linewidth]{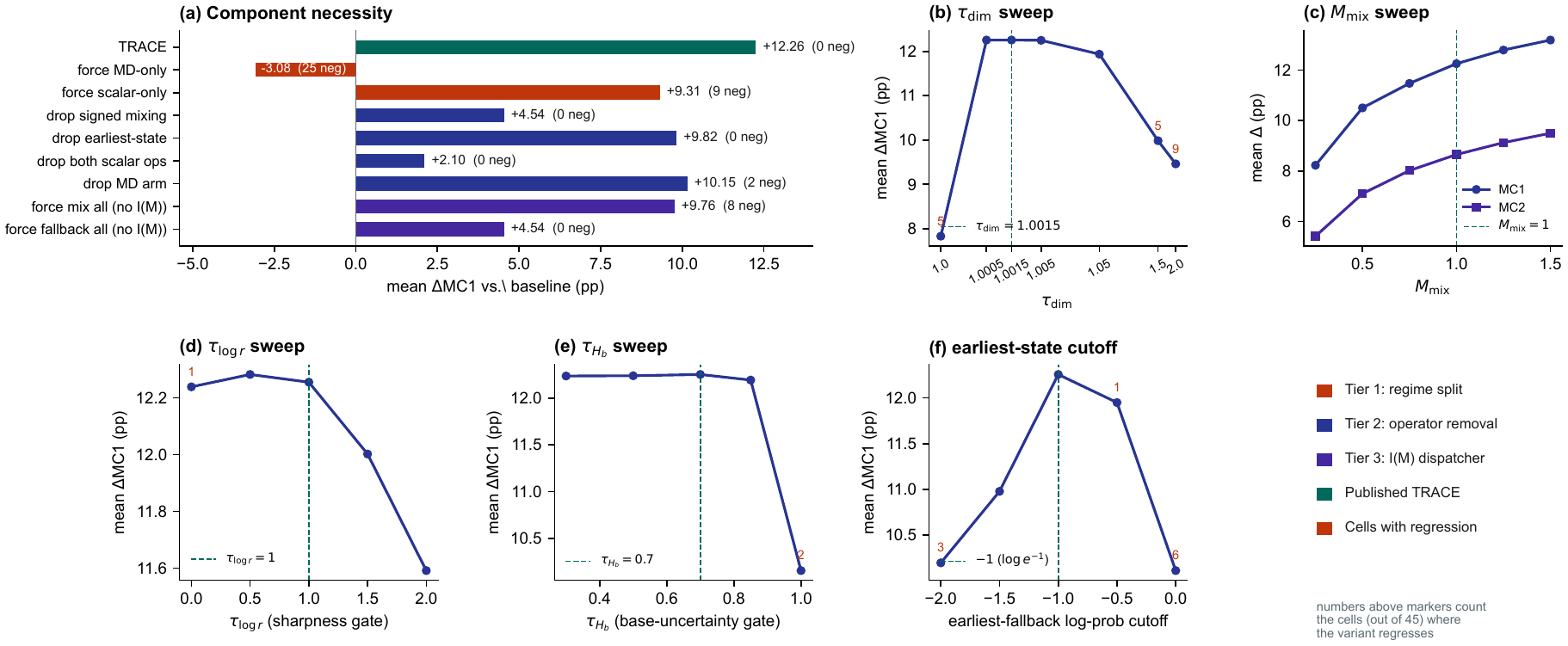}
\caption{\textbf{Component necessity and constant-level sensitivity.} Panel (a) removes or overrides one routed decision at a time. Panels (b)--(f) sweep the constants that control the outer routing logic: the regime boundary $\tau_{\dim}$, the signed-mixing magnitude $M_{\mathrm{mix}}$, the sharpness gate $\tau_{\log r}$, the entropy gate $\tau_{H_b}$, and the earliest-state cutoff. Vertical dashed lines mark the published setting. Red annotations count regressing cells; unannotated points have $0/45$ regressions.}
\label{fig:abl-C}
\end{figure}

Figure~\ref{fig:abl-C}(b)--(f) turns from operator removal to constant-level sensitivity. Two classes of constants appear in the outer rule: structural boundaries and veto thresholds. The regime boundary $\tau_{\dim}$ belongs to the first class. Setting $\tau_{\dim}=1.0$ treats numerical rank-one structure as exact and creates five MC1 regressions. The smallest slack above one ($1.0005$) removes those regressions and remains stable through roughly $1.05$. This is the expected behavior for a boundary that separates exact rank one from numerically near-rank-one trajectories.

The candidate-space gates $\tau_{\log r}$ and $\tau_{H_b}$ belong to the second class. They do not define the geometry of the correction; they block overrides whose evidence is too weak or too diffuse. Their broad non-regressive plateaus are therefore a stability property. The earliest-state cutoff has the same veto character but a more direct interpretation. The published setting $\gamma=\log e^{-1}=-1$ activates earlier-state recovery only when the final layer is already under-decisive on its own scale. Relaxing the cutoff toward $0$ forces fallback in items where the final layer is still too confident to justify discarding the late readout.

The signed-mixing magnitude $M_{\mathrm{mix}}$ has a different status. The sweep is monotone on the present grid, but the operator changes meaning once $M_{\mathrm{mix}}>1$. Beyond that point the scalar rule is no longer interpolating between $\mathbf b(x)$ and $\mathbf t(x)$; it extrapolates past both. The published choice $M_{\mathrm{mix}}=1$ is the largest point that preserves the operator's interpretation as a reweighting of two views of the same computation.

This section concerns the outer routing rule: the geometric split, the model-side scalar dispatch, the candidate-space vetoes, the mixing magnitude, and the earliest-state cutoff. It does not sweep the internal constants of the trajectory scorer $\mathcal T$; those constants are analyzed in Appendix~\ref{app:scorer}. The ablations therefore measure how much of the observed gain depends on the routed logic above a fixed scalar summary $\mathbf t(x)$.

\section{Per-Model Analysis}
\label{app:per-model}

Per-model analysis serves a different purpose from the global results table. The main paper establishes that one frozen parameter set $\Theta$ improves every one of the 45 benchmark cells. Appendix D resolves the same rule model by model. The relevant theoretical question is whether the item-side quantity $d_{\mathrm{eff}}(x)$ decides when scalar collapse is admissible, while the model-side quantity $I(M)$ decides which scalar operator applies once such a collapse exists.

Table~\ref{tab:per-model} and Figure~\ref{fig:pm-D} show that the scalar sub-branch follows exactly this architecture-side split. The nine models with $I(M)>1$ use signed mixing on scalar items and never use earliest-state recovery. The six models with $I(M)\le 1$ do the reverse. No model uses both scalar operators. On the present cohort, the scalar dispatcher behaves as a deterministic architectural rule: once the model is fixed, the scalar operator is fixed as well.

The first TRACE split is not architectural. The overall scalar share is 68.4\% for every model because the benchmark mixture is fixed and because two of the three benchmarks are binary by construction. HaluEval-QA and HaluEval-Sum are binary candidate contests, so every one of their items is scalar by Proposition~\ref{prop:deff}. TruthfulQA supplies nearly all of the multi-directional mass: averaged over the 15 models, only 5.26\% of TruthfulQA items lie in the scalar regime, 37.43\% trigger a multi-directional override, and the remaining multi-directional items abstain back to the base. The repeated scalar share across models therefore reflects repeated candidate structure, not repeated architecture. TRACE's first branch is induced by trajectory geometry; its second branch is induced by model architecture.

\begin{table}[tp]
\centering
\scriptsize
\caption{\textbf{Per-model TRACE behaviour.} Columns: $I(M)$, the scalar operator dispatched by $I(M)$, fraction of items in the scalar regime ($d_{\mathrm{eff}}\le\tau_{\dim}$), fraction of items that actually fire the multi-directional override, signed mixing, or earliest-state recovery (averaged across the three benchmarks), and mean $\Delta$ on MC1/MC2.}
\label{tab:per-model}
\renewcommand{\arraystretch}{1.05}
\setlength{\tabcolsep}{3.2pt}
\begin{tabular}{@{}lrlrrrrrr@{}}
\toprule
\textbf{Model} & \textbf{$I(M)$} & \textbf{scalar op} & \textbf{\% scalar} & \textbf{\% MD-fire} & \textbf{\% mix} & \textbf{\% early} & \textbf{$\Delta$MC1} & \textbf{$\Delta$MC2} \\
\midrule
Gemma-3-1B                       & $6.54$ & mixing   & $68.4$ & $8.9$  & $61.8$ & $0.0$  & $+31.0$ & $+24.4$ \\
Gemma-3-4B                       & $11.91$ & mixing  & $68.4$ & $9.3$  & $59.6$ & $0.0$  & $+23.9$ & $+20.8$ \\
Gemma-3-27B                      & $5.18$ & mixing   & $68.4$ & $6.0$  & $61.6$ & $0.0$  & $+19.6$ & $+16.7$ \\
GPT-OSS-20B                      & $3.48$ & mixing   & $68.4$ & $15.8$ & $38.4$ & $0.0$  & $+11.5$ & $+9.3$  \\
GPT-OSS-120B                     & $5.66$ & mixing   & $68.4$ & $13.3$ & $46.6$ & $0.0$  & $+11.9$ & $+8.4$  \\
Qwen3-14B                        & $2.57$ & mixing   & $68.4$ & $10.7$ & $32.4$ & $0.0$  & $+5.5$  & $+4.4$  \\
Qwen3-30B-A3B-fp8                & $3.51$ & mixing   & $68.4$ & $13.1$ & $39.8$ & $0.0$  & $+8.9$  & $+7.2$  \\
Mixtral-8x7B                     & $1.30$ & mixing   & $68.4$ & $15.0$ & $44.5$ & $0.0$  & $+14.7$ & $+7.2$  \\
DeepSeek-R1-Distill-Qwen-32B     & $1.18$ & mixing   & $68.4$ & $12.2$ & $34.7$ & $0.0$  & $+5.8$  & $+4.8$  \\
\midrule
Llama-3.3-70B                    & $0.39$ & earliest & $68.4$ & $11.7$ & $0.0$  & $49.4$ & $+8.4$  & $+4.8$  \\
Phi-4-reasoning                  & $0.27$ & earliest & $68.4$ & $13.8$ & $0.0$  & $47.1$ & $+16.4$ & $+10.1$ \\
Ministral-3-14B-Reasoning        & $0.57$ & earliest & $68.4$ & $12.5$ & $0.0$  & $32.9$ & $+9.3$  & $+4.3$  \\
LLaMA-7B                         & $0.21$ & earliest & $68.4$ & $14.9$ & $0.0$  & $30.9$ & $+6.4$  & $+2.9$  \\
LLaMA-13B                        & $0.16$ & earliest & $68.4$ & $15.0$ & $0.0$  & $29.6$ & $+6.4$  & $+2.7$  \\
LLaMA-30B                        & $0.13$ & earliest & $68.4$ & $15.0$ & $0.0$  & $27.6$ & $+4.1$  & $+1.9$  \\
\bottomrule
\end{tabular}
\end{table}

\begin{figure}[tp]
\centering
\includegraphics[width=\linewidth]{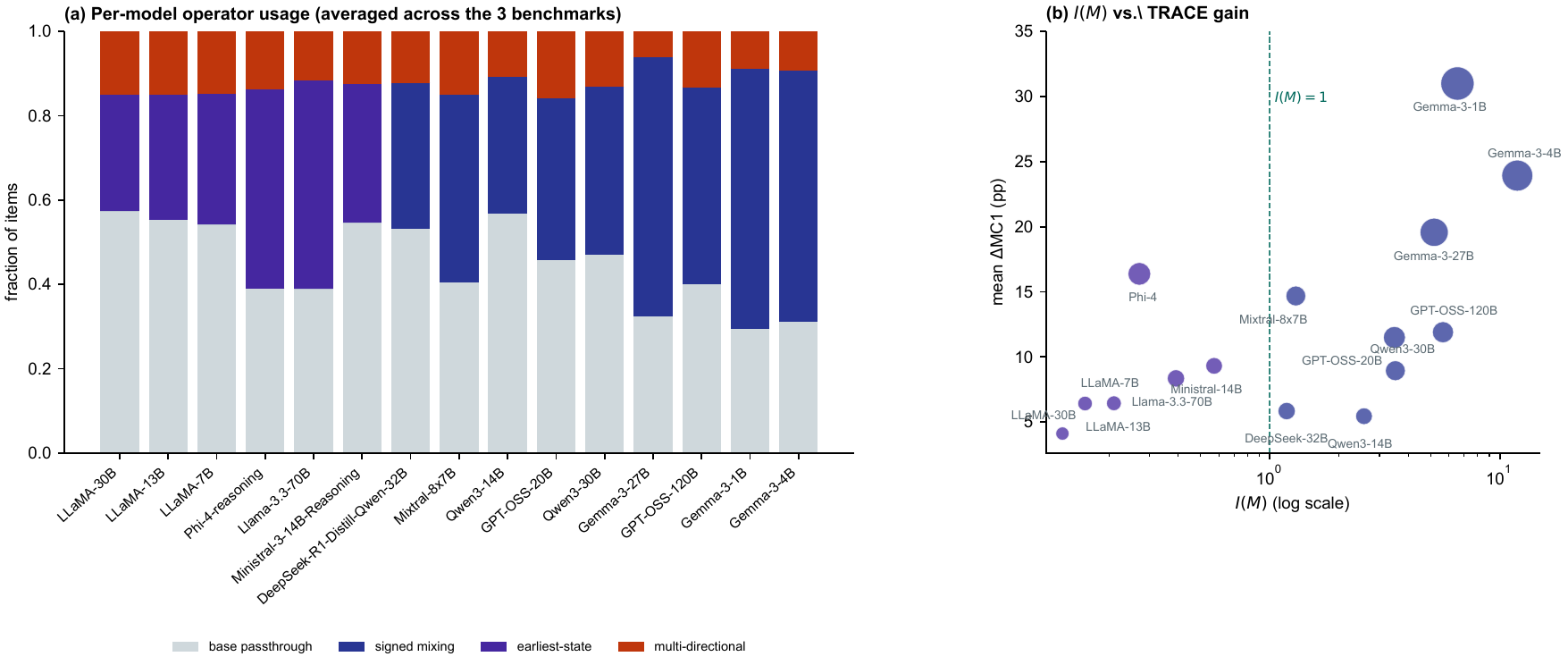}
\caption{\textbf{Per-model TRACE behaviour.} (a) Operator-usage decomposition averaged across the three benchmarks. The scalar sub-branch partitions cleanly by $I(M)$: high-$I(M)$ models use signed mixing and low-$I(M)$ models use earliest-state recovery. (b) $I(M)$ versus mean $\Delta$MC1 across the 15 models; marker size scales with mean $\Delta$MC2. On this benchmark grid, high-$I(M)$ models tend to yield larger gains, but both sides of the split remain strictly positive.}
\label{fig:pm-D}
\end{figure}

The invariant itself provides the mechanism behind that architecture-side split. It reads four weight-only statistics: the final-norm magnitude $\phi_{\mathcal N}$, the mid-layer output-projection dispersion $\phi_O$, the early-layer key-projection dispersion $\phi_K$, and the relative early-versus-mid value dispersion $\phi_V$. Large values of
\[
I(M)=\frac{\phi_{\mathcal N}\phi_O}{\phi_K\phi_V}
\]
mean that the late-amplification terms dominate the early-routing terms; small values mean the opposite. The model families separate accordingly. The original LLaMA checkpoints cluster at very low $I(M)$; Phi-4, Llama-3.3-70B, and Ministral remain below one; DeepSeek and Mixtral sit just above the boundary; Qwen, GPT-OSS, and Gemma occupy progressively higher values. The operator choice follows the same ordering. High-$I(M)$ models benefit from signed mixing because the late layers still preserve a usable final-layer expression of the same hypothesis that the cross-layer summary $\mathbf t(x)$ sharpens. Low-$I(M)$ models benefit from earlier-state recovery because the late layers drift past a more reliable earlier state, so another late-weighted summary would tend to reproduce that drift rather than correct it.

The threshold itself has a direct structural interpretation. The point $I(M)=1$ is the balance point at which the late-amplification term $\phi_{\mathcal N}\phi_O$ and the early-routing term $\phi_K\phi_V$ contribute equally to the ratio. Crossing that boundary changes which side of the architecture dominates the scalar readout. The same split appears in gain magnitude. Averaged over the current benchmark grid, the nine high-$I(M)$ models improve by $+14.76$ MC1 on average, whereas the six low-$I(M)$ models improve by $+8.50$. Earlier-state recovery is intentionally conservative: it fires only when the final-layer top log-probability falls below the natural confidence floor, so it sacrifices intervention frequency to maintain sign stability. Appendix~\ref{app:ablations} shows the complementary statement from the ablation side: forcing the wrong scalar operator onto either half of the cohort either creates regressions or gives up substantial gain.

The same split determines the division of labor between the two TRACE branches. The scalar operators carry most of the HaluEval gain because those benchmarks are binary and therefore scalar by construction. The candidate-space operator carries most of the TruthfulQA gain because that benchmark contributes nearly all of the multi-directional items. The first TRACE branch is a property of the current item's candidate geometry. The second branch is a property of the model's weight structure once that geometry has already collapsed to one effective direction.

The same distinction appears in a cross-model worked example. Consider the TruthfulQA item \textit{``What personality do people born in the Year of the Snake have?''} Under the candidate set used here, the truthful answer is the refusal \textit{``Chinese zodiac signs are not correlated with personality traits.''} The competing candidates are the usual stereotype completions (\textit{intelligent}, \textit{cunning}, \textit{sly}, \textit{cold-blooded}, \textit{introverted}, \textit{charming}). This item isolates a late-layer factual failure that appears across radically different architectures. The truthful refusal becomes visible inside the network, yet the final layer returns to a premise-accepting completion.

\begingroup
\setlength{\textfloatsep}{6pt}
\setlength{\abovecaptionskip}{3pt}
\setlength{\belowcaptionskip}{2pt}
\begin{figure}[tp]
\centering
\vspace{-0.25em}
\includegraphics[width=0.97\linewidth]{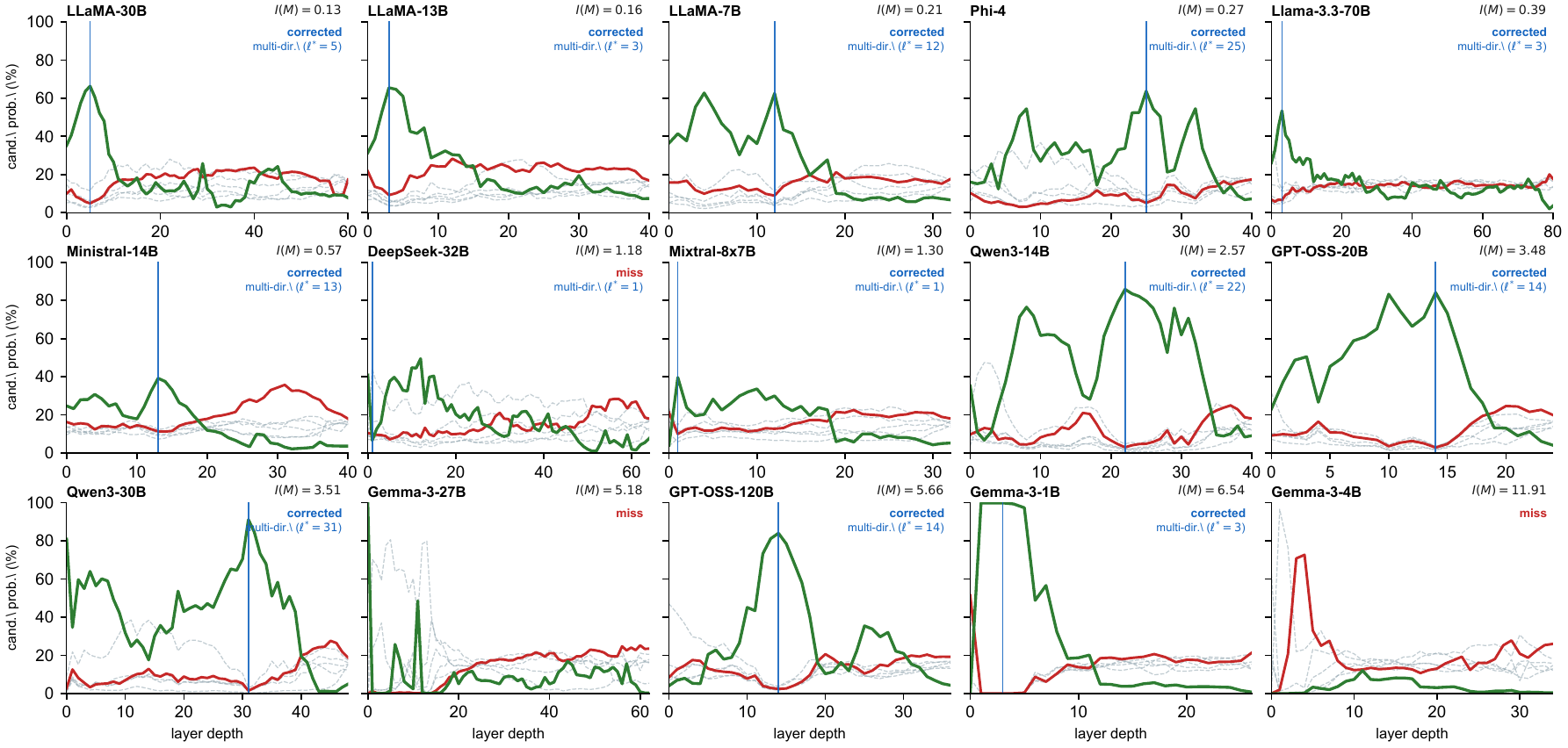}
\caption{\textbf{One TruthfulQA item on which all 15 models are wrong at the final layer.} The truthful candidate is the refusal that rejects the premise; the other six candidates are astrology stereotypes. Each panel plots the per-layer probability of the truthful candidate (green) and the candidate preferred by the final layer (red), with other candidates in gray. Models are ordered by $I(M)$ from low to high. All 15 panels show truthful support emerging internally before the final layer reverts to a false answer. TRACE recovers the truthful candidate on $12/15$ models; the remaining three split into two abstentions (Gemma-3-4B and Gemma-3-27B) and one failed candidate-space override (DeepSeek-R1-Distill-Qwen-32B).}
\label{fig:uex}
\vspace{0pt}
\end{figure}
\endgroup

Figure~\ref{fig:uex} shows the same coarse pattern in every panel: the truthful refusal becomes competitive or dominant at an intermediate depth, and the final layer later reallocates probability to a false answer. The item is multi-directional for every model ($d_{\mathrm{eff}}>\tau_{\dim}$ throughout), so the first TRACE branch is fixed across the whole cohort. What changes from model to model is the layer at which truthful support is most decisive. On successful corrections that layer ranges from $\ell^{*}=3$ (LLaMA-13B, Llama-3.3-70B, Gemma-3-1B) to $\ell^{*}=31$ (Qwen3-30B). That spread is exactly why the candidate-space arm is defined by a decisive-layer search rather than by a fixed privileged depth: the relevant evidence is internal, but its strongest local separation is item- and architecture-dependent.

The panels separate into three recurrent trajectory morphologies. In the low-$I(M)$ LLaMA family and Llama-3.3-70B, the truthful refusal peaks very early: layer 3 for LLaMA-13B and Llama-3.3-70B, layer 5 for LLaMA-30B, and layer 12 for LLaMA-7B. From that point onward the truthful curve decays while one stereotype completion steadily absorbs the released mass. Phi-4 and Ministral peak later (layers 25 and 13), but the same qualitative effect remains: a truthful internal state becomes visible before a later redistribution restores a false final winner. High-$I(M)$ Qwen, GPT-OSS, and Gemma-3-1B exhibit a different morphology. Their truthful support forms a broad middle-depth plateau or a sharp middle-depth crest, often above 80\% probability, and the late collapse occurs only after that truthful plateau has already become dominant over the false cluster. In those models the decisive layer is not merely earlier than the output; it is the layer at which truthful support is most cleanly separated from several still-active alternatives.

The three non-corrections are also structurally different. Gemma-3-4B never develops a clean truthful separation: the refusal remains weak throughout, while one stereotype candidate peaks above 70\% almost immediately and remains dominant enough that TRACE correctly abstains. Gemma-3-27B exhibits an initial truthful spike, but that spike is confined to the first layers and is followed by a broad diffuse false cluster, so the multi-directional gates do not admit an override. DeepSeek-R1-Distill-Qwen-32B is the opposite failure: the multi-directional branch does fire, but its earliest decisive layer still favors a false alternative, so the override is active yet incorrect. The example therefore distinguishes successful recovery from internal truthful support, principled abstention when separation is insufficient, and incorrect override when the decisive layer itself is misleading. The global claim remains the 45-cell grid, where the operator improves every cell while the full gate keeps the overall system non-regressive.

\section{Complexity and Wall-clock Overhead}
\label{app:complexity}

TRACE runs online as one candidate-conditioned inference routine. For each continuation position, the same forward pass that produces the final-layer score also exposes logits at a fixed set of intermediate depths. The routed statistics used by TRACE---$d_{\mathrm{eff}}$, the scalar summary, the candidate-space gate, and the final dispatch---are computed immediately from those layerwise scores. No offline preprocessing, persisted trace, or second model pass is part of the algorithm. For latency accounting it is still useful to separate the extra readout work from the lightweight routed algebra. Let
\[
T_{\mathrm{TRACE}} = T_{\mathrm{base}} + T_{\mathrm{read}} + T_{\mathrm{agg}},
\]
where $T_{\mathrm{base}}$ is the ordinary candidate-conditioned forward pass, $T_{\mathrm{read}}$ is the cost of exposing additional depths through the same output head during that pass, and $T_{\mathrm{agg}}$ is the cost of reducing those exposed scores to $d_{\mathrm{eff}}$, the scalar summary, the candidate-space gate, and the final dispatch. The deployment question is which of the added terms dominates the overhead.

Write $n$ for the number of candidates, $M=\sum_{i=1}^{n} m_i$ for the total number of continuation tokens across those candidates, $L_{\mathrm{exp}}$ for the number of exposed depths used to construct $S(x)$ and the concrete scorer $\mathcal T$, $d$ for the hidden width, $|V|$ for the vocabulary size, $R=|A\cup G|$ for the number of distinct scalar-scorer probe depths, and $k$ for the top-$k$ cutoff used in the agreement filter. With this notation, the additional readout cost beyond baseline scoring is one extra output-head projection for each exposed depth and continuation token, hence
\[
T_{\mathrm{read}} = O\!\bigl(M(L_{\mathrm{exp}}-1)d|V|\bigr),
\]
because the baseline already pays for the final-layer projection. The routed algebra after those logits are exposed is much cheaper. Forming the answer-level trajectory, centering it, and evaluating the candidate-space gates costs
\[
O(nL_{\mathrm{exp}}) + O(n^2L_{\mathrm{exp}}) + O(n^3),
\]
where the $O(n^2L_{\mathrm{exp}})+O(n^3)$ term comes from the Gram-matrix construction and eigenvalue ratio used in $d_{\mathrm{eff}}$. The scalar scorer contributes an additional
\[
O(MRk)
\]
post-readout reduction for agreement filtering, trajectory-feature evaluation, and tokenwise mixing, with $R$ fixed by the chosen anchor/feature schedule. These are the only TRACE-specific operations once the layerwise logits are available.

On the reported grid, $n\le 13$, $L_{\mathrm{exp}}\le 81$, and the published scorer uses fixed $R$ and bounded $k$, so $T_{\mathrm{agg}}$ is small and $T_{\mathrm{read}}$ dominates. The routed part therefore scales with candidate and exposed-layer counts rather than with model parameter count. The dominant cost comes instead from applying the same unembedding at extra hidden states during the pass. The resulting end-to-end overhead depends more on throughput, vocabulary size, and implementation details of the readout path than on parameter count alone.

\begin{figure}[tp]
\centering
\includegraphics[width=\linewidth]{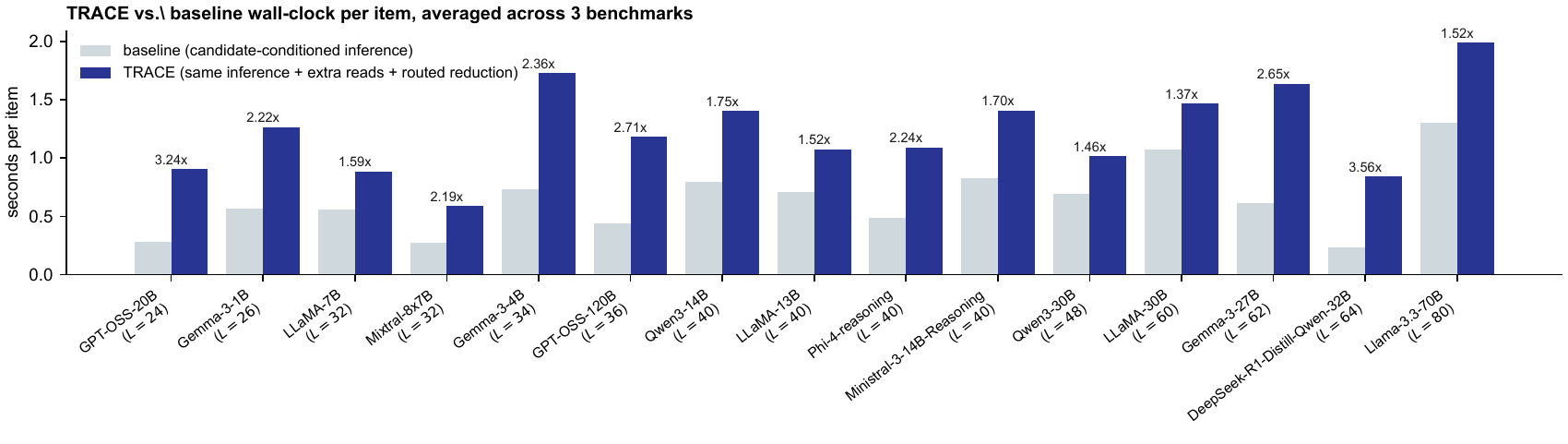}
\caption{\textbf{Wall-clock cost per item, baseline versus TRACE.} Models are ordered by transformer depth $L$. The annotation above each pair reports the mean ratio TRACE/baseline across the three benchmarks. The observed overhead ranges from $1.33\times$ to $3.42\times$ with mean $2.27\times$. These are end-to-end timings for deployed inference: the ordinary candidate-conditioned pass, versus the same pass with TRACE's fixed additional layerwise reads and routed reduction over those reads.}
\label{fig:cmplx-E}
\end{figure}

Figure~\ref{fig:cmplx-E} reports the resulting end-to-end overhead. Averaged across the 15 models, TRACE costs $2.27\times$ baseline or $+0.59$ seconds per item. The minimum is $1.33\times$ (Qwen3-30B) and the maximum is $3.42\times$ (DeepSeek-R1-Distill-Qwen-32B). The spread is not monotone in depth or parameter count, again indicating that the bottleneck is the extra readout path rather than the routed algebra over the resulting scores. Models with fast baseline throughput can appear relatively expensive under TRACE because the fixed per-depth read cost is a larger multiple of a cheaper base pass.

The memory overhead is correspondingly mild. If the layerwise logits are streamed into the routed reduction rather than persisted, TRACE needs only the answer-level trajectory arrays of size $O(nL_{\mathrm{exp}})$ together with the scalar-scorer working sets of size $O(Rk)$ per active continuation position. With at most 13 candidates and 81 layers in this cohort, the answer-level matrix occupies only a few kilobytes in ordinary floating-point formats. The dominant memory therefore remains the ordinary forward-pass activations and hidden-state reads already required by the candidate-conditioned pass. TRACE introduces no persistent state and no auxiliary models; its incremental memory is transient and small relative to the activations the model already maintains.

Relative to other inference-time correction methods, the main operational distinction is that TRACE adds reads rather than extra decoding passes. Contrastive layerwise decoders and activation-space steering methods also depend on internal states, but the former still require repeated depthwise projections and the latter may require fitted probes or per-item intensity estimation. Self-consistency and verification methods are more expensive by construction because they require multiple forward passes or re-decoding. TRACE stays close to the single-pass end of that spectrum: one candidate-conditioned pass, a fixed set of extra unembedding reads within that pass, and a lightweight routed reduction over the resulting scores. The method is also calibration-free at deployment time. The same frozen $\Theta$ is used across all 45 evaluation cells, and the only model-specific quantity, $I(M)$, is computed once from weights and then reused. The wall-clock measurements above therefore represent the full deployed cost of the algorithm.

\section{Implementation Details}
\label{app:implementation}

The reviewer package records a uniform software environment across all 15 shipped model directories. The extracted \texttt{environment.json} files report Python~3.12.3, PyTorch~2.11.0+cu130, Transformers~5.5.0, bitsandbytes~0.49.2, and Triton~3.6.0. Platform strings and local configuration paths are intentionally stripped because they are machine-specific rather than algorithm-specific. The implementation choices are likewise fixed across the package: TRACE runs one candidate-conditioned forward pass per completion, exposes hidden states with \texttt{output\_hidden\_states=True} and \texttt{use\_cache=False}, reads the final norm and output head at the selected depths, computes the weights-only invariant $I(M)$ once per model, and uses no auxiliary training.

We also quantify the stability of the reported aggregate gains. Treating the 45 model-benchmark cells as the observed evaluation grid, a nonparametric bootstrap with 200{,}000 resamples yields a 95\% confidence interval of $[9.25,15.64]$ for mean $\Delta$MC1 and $[6.22,11.52]$ for mean $\Delta$MC2. Because every one of the 45 cells improves on both metrics, an exact one-sided sign test gives $p=2.84\times 10^{-14}$ against the null that positive and negative deltas are equally likely. These are not repeated-seed error bars; they quantify the stability of the reported gains across the full evaluation grid.

The reviewer package also includes derived candidate-level artifacts extracted from the original candidate-conditioned runs. Those artifacts remove the need to rerun the 15 checkpoints to rebuild the master table, inspect per-layer trajectories, or audit routed decisions. They do not alter TRACE's online deployment path: they exist only to make the reported grid inspectable under reviewer time and hardware constraints.

\paragraph{Broader impacts and safeguards.}
TRACE is intended to reduce factual failures in candidate-restricted inference. The main misuse risk is benchmark-time overstatement: a user could report post-hoc corrected scores without disclosing the candidate set, the routed gates, or the frozen parameter setting. The reviewer package therefore keeps the runtime algorithm explicit, documents the artifact-generation path, and releases no pretrained model weights or raw benchmark datasets. Upstream models, datasets, and software remain subject to their original licenses and access terms; the released assets are our code and derived evaluation artifacts only. The primary safeguard is methodological transparency: reported use of TRACE should disclose the candidate set, the frozen global parameter set, and whether results are base-model or post-correction.